%% file: main.tex
\pdfoutput=1

\documentclass{article}

\PassOptionsToPackage{table}{xcolor}

\usepackage[preprint]{neurips_2026}

\usepackage[utf8]{inputenc}
\usepackage[T1]{fontenc}
\usepackage[hyphens]{url}
\usepackage{booktabs}
\usepackage{graphicx}
\usepackage{amsmath}
\usepackage{amssymb}
\usepackage{nicefrac}
\usepackage{microtype}
\usepackage{xcolor}

\usepackage{amsthm}
\usepackage{subcaption}
\usepackage{float}
\usepackage{wrapfig}
\usepackage{colortbl}
\usepackage{xspace}
\usepackage[breakable,skins]{tcolorbox}
\usepackage{enumitem}
\usepackage{tabularx}
\usepackage{titletoc}
\usepackage{algorithm}
\usepackage{algorithmic}
\usepackage{hyperref}

\input{macros}

\input{boxes}
\hypersetup{colorlinks=true, linkcolor=mylinkcolor,
            citecolor=mylinkcolor, urlcolor=mylinkcolor}

\input{sections/tables/slots}

\title{ERRAND: Budgeted Maintenance of Agent Memory}

\author{%
  Beining Wu \quad Zihao Ding \quad Jun Huang\\
  Department of Electrical Engineering and Computer Science\\
  South Dakota State University, Brookings, SD 57007, USA\\
  \texttt{\{Wu.Beining, Zihao.Ding\}@jacks.sdstate.edu}\\
  \texttt{Jun.Huang@sdstate.edu}\\
}

\begin{document}
\raggedbottom

\maketitle

\input{sections/0_abstract}

\begin{figure}[H]
\centering
\includegraphics[width=0.377\linewidth]{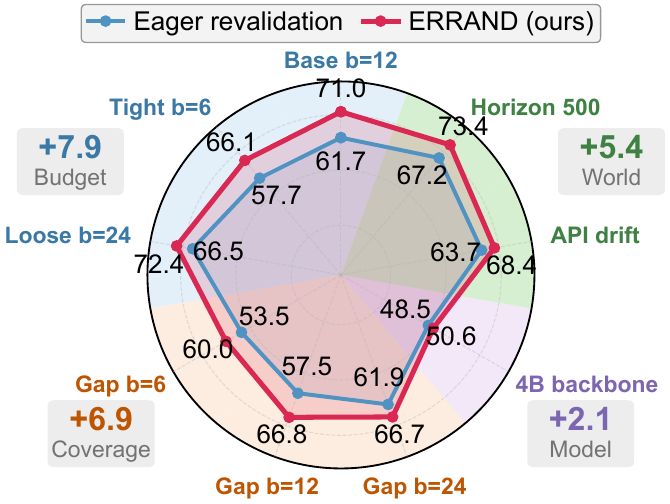}\hfill%
\includegraphics[width=0.595\linewidth]{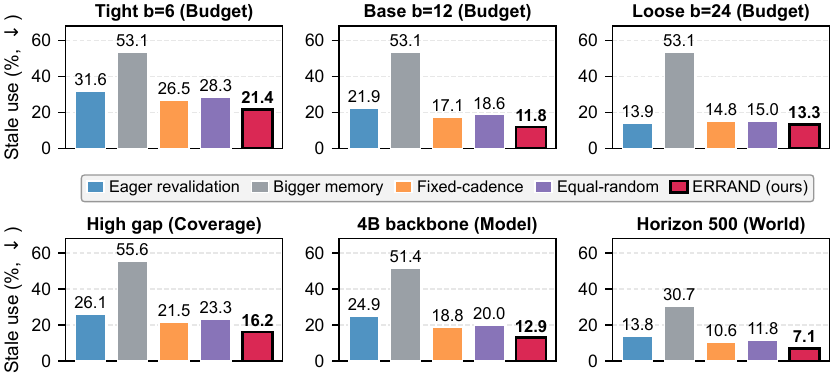}
\caption{\textbf{Priced revalidation beats spending the same budget without
prices, everywhere we deploy it.} \textbf{Left:} success across nine
deployment axes; badges: mean gain per axis family. \textbf{Right:}
stale-use share ($\downarrow$) across six settings, \sys{} lowest in all.
One protocol: frozen agent, drifting world, hard action caps; oracle arms
are priced in Table~\ref{tab:main}.}
\label{fig:teaser}
\end{figure}

\input{sections/1_intro}
\input{sections/2_related}
\input{sections/3_method}
\input{sections/4_experiments}
\input{sections/5_conclusion}

{\small
\bibliographystyle{plainnat}
\bibliography{references}
}

\appendix
\setlength{\textfloatsep}{8pt plus 2pt minus 2pt}
\setlength{\intextsep}{6pt plus 2pt minus 2pt}
\setlength{\floatsep}{8pt plus 2pt minus 2pt}
\makeatletter
\setlength{\@neuripsabovecaptionskip}{4\p@}
\makeatother
\setlength{\abovecaptionskip}{4pt}
\input{sections/6_appendix}

\end{document}

%% file: macros.tex
\newcommand{\sys}{\textsc{Errand}\xspace}

\newcommand{\wage}{\hat{\nu}}

\definecolor{delneg}{RGB}{214,60,61}
\definecolor{delpos}{RGB}{58,121,165}

\theoremstyle{definition}

\newtheorem{proposition}{Proposition}

\newcommand{\StaleCost}{\slot{StaleCost}}
\newcommand{\FrontierGap}{\slot{FrontierGap}}

%% file: boxes.tex
\definecolor{mylinkcolor}{RGB}{11, 20, 110}

\definecolor{DeltaBg}{HTML}{D4F2D7}
\definecolor{SearchBg}{HTML}{C2E6F5}
\definecolor{AgenticBg}{HTML}{F5C2CC}
\definecolor{MathBg}{HTML}{E6D4F2}
\definecolor{ScienceBg}{HTML}{FBE0BC}

\definecolor{darkgreen}{RGB}{0, 181, 18}
\definecolor{darkred}{RGB}{252, 90, 90}
\newcommand{\mygreen}[1]{\cellcolor{darkgreen!#1}}
\newcommand{\myredcell}[1]{\cellcolor{darkred!#1}}

\definecolor{toolcardbox}{RGB}{240, 248, 255}
\definecolor{toolcardborder}{RGB}{52, 52, 173}
\definecolor{mygray}{RGB}{242, 242, 242}
\definecolor{codegray}{gray}{0.95}
\definecolor{failbg}{RGB}{248, 230, 234}
\definecolor{failframe}{RGB}{176, 36, 24}
\definecolor{failbadge}{RGB}{225, 151, 168}
\definecolor{successbg}{RGB}{239, 255, 229}
\definecolor{successframe}{RGB}{34, 139, 34}
\definecolor{successbadge}{RGB}{182, 200, 108}

\newtcolorbox{summarybox}{%
    colback=toolcardbox,
    colframe=toolcardborder,
    arc=3pt,
    boxrule=0.8pt,
    left=2pt, right=2pt, top=2pt, bottom=2pt,
    fonttitle=\bfseries,
    breakable,
}

\newtcolorbox{limitationbox}{%
    colback=failbg,
    colframe=failframe,
    arc=1pt,
    boxrule=0.8pt,
    left=4pt, right=4pt, top=4pt, bottom=4pt,
    fonttitle=\bfseries\small,
    title={Limitation},
    coltitle=white,
    fontupper=\small,
    breakable,
    before skip=6pt, after skip=6pt
}

\newtcolorbox{bestpracticebox}{%
    colback=successbg,
    colframe=successframe,
    arc=1pt,
    boxrule=0.8pt,
    left=4pt, right=4pt, top=4pt, bottom=4pt,
    fonttitle=\bfseries\small,
    title={Best Practice},
    coltitle=white,
    fontupper=\small,
    breakable,
    before skip=6pt, after skip=6pt
}

\newtcolorbox{custombox}[2][]{
    colback=toolcardbox,
    colframe=toolcardborder,
    coltitle=white,
    arc=1pt,
    boxrule=1pt,
    fonttitle=\bfseries,
    title={#2},
    left=5pt, right=5pt, top=5pt, bottom=5pt,
    before skip=1em, after skip=1em,
    fontupper=\small,
    breakable,
    width=1.\linewidth,
    #1
}

\newtcolorbox{casebox}[3]{
    colback=#1,
    colframe=#2,
    arc=1pt,
    coltitle=white,
    fonttitle=\bfseries,
    title={#3},
    boxrule=1pt,
    rounded corners,
    breakable
}
\newtcolorbox{casecode}[1]{%
    colback=#1,
    colframe=toolcardbox,
    arc=1pt,
    boxrule=0pt,
    top=-3pt, bottom=0pt,
    left=3pt, right=3pt,
    boxsep=0pt,
    fontupper=\ttfamily\fontsize{8.0}{10}\selectfont,
    breakable,
}
\newtcolorbox{simplecode}[1][]{%
    colback=toolcardbox,
    colframe=toolcardbox,
    arc=1pt,
    boxrule=0pt,
    top=-3pt, bottom=0pt,
    left=3pt, right=3pt,
    boxsep=0pt,
    fontupper=\ttfamily\fontsize{8.0}{10}\selectfont,
    breakable,
    #1
}

\definecolor{zebra}{gray}{0.93}
\definecolor{herobg}{RGB}{230,239,246}
\definecolor{grpblue}{RGB}{228,238,246}
\definecolor{grpteal}{RGB}{228,242,240}
\definecolor{grppurple}{RGB}{236,232,243}%
\definecolor{grporange}{RGB}{252,240,229}%
\newcommand{\stdv}[1]{$_{\pm\text{\scriptsize #1}}$}

\newtcolorbox{envcardbox}[1]{breakable, enhanced, colback=toolcardbox,
  colframe=toolcardborder, coltitle=white, fonttitle=\bfseries\small,
  title={Environment Card: #1}, rounded corners, boxrule=1pt,
  left=5pt, right=5pt, top=4pt, bottom=4pt, fontupper=\small}

\definecolor{bandblue}{RGB}{232,239,249}
\definecolor{bandgreen}{RGB}{233,243,232}
\definecolor{bandorange}{RGB}{251,239,225}
\definecolor{bandpurple}{RGB}{242,235,244}
\definecolor{bandred}{RGB}{248,234,232}
\newcommand{\Comment}[1]{\hfill$\triangleright$ \textit{#1}}

%% file: sections/tables/slots.tex
\input{sections/tables/tabs_slots}

\input{sections/tables/fig1_slots}
\input{sections/tables/fig3_slots}

\input{sections/tables/fig4_slots}
\input{sections/tables/fig5_slots}
\input{sections/tables/fig6_slots}
\input{sections/tables/app_slots}

%% file: sections/tables/tabs_slots.tex
\providecommand{\TDeltaZero}{}\renewcommand{\TDeltaZero}{7.1}

\providecommand{\AbCalibDrop}{}\renewcommand{\AbCalibDrop}{13.0}
\providecommand{\AbDeadbandDrop}{}\renewcommand{\AbDeadbandDrop}{3.5}
\providecommand{\AbRegimeDrop}{}\renewcommand{\AbRegimeDrop}{5.0}
\providecommand{\AbComparatorDrop}{}\renewcommand{\AbComparatorDrop}{3.1}

\providecommand{\OvCallsOurs}{}\renewcommand{\OvCallsOurs}{1.10}
\providecommand{\OvCallsZero}{}\renewcommand{\OvCallsZero}{1.40}

%% file: sections/tables/fig3_slots.tex
\providecommand{\FrontierBTwoGain}{}\renewcommand{\FrontierBTwoGain}{4.5}

\providecommand{\NuInfFloor}{}\renewcommand{\NuInfFloor}{51.1}

%% file: sections/tables/fig4_slots.tex
\providecommand{\StaleCost}{}\renewcommand{\StaleCost}{79.5}
\providecommand{\StaleFresh}{}\renewcommand{\StaleFresh}{85.8}
\providecommand{\StaleStale}{}\renewcommand{\StaleStale}{6.3}
\providecommand{\FrontierGap}{}\renewcommand{\FrontierGap}{10.0}
\providecommand{\BudgetGapTight}{}\renewcommand{\BudgetGapTight}{8.0}
\providecommand{\BudgetGapBase}{}\renewcommand{\BudgetGapBase}{10.0}
\providecommand{\BudgetGapLoose}{}\renewcommand{\BudgetGapLoose}{5.2}
\providecommand{\BBDropItt}{}\renewcommand{\BBDropItt}{4.7}

\providecommand{\SelfCapSpend}{}\renewcommand{\SelfCapSpend}{11.0}
\providecommand{\EagerNocapSpend}{}\renewcommand{\EagerNocapSpend}{70.7}

\providecommand{\SpendMultiple}{}\renewcommand{\SpendMultiple}{6.4}
\providecommand{\HeroBaseItt}{}\renewcommand{\HeroBaseItt}{71.9}
\providecommand{\EagerBaseItt}{}\renewcommand{\EagerBaseItt}{61.9}
\providecommand{\FrontierGapHigh}{}\renewcommand{\FrontierGapHigh}{9.9}
\providecommand{\CondGapBase}{}\renewcommand{\CondGapBase}{1.6}
\providecommand{\CondGapHigh}{}\renewcommand{\CondGapHigh}{3.9}
\providecommand{\EagerNocapTrail}{}\renewcommand{\EagerNocapTrail}{4.5}

%% file: sections/tables/fig5_slots.tex
\providecommand{\CorrflipExcessLo}{}\renewcommand{\CorrflipExcessLo}{0.6}
\providecommand{\CorrflipExcessHi}{}\renewcommand{\CorrflipExcessHi}{3.6}

\providecommand{\CalibStepCal}{}\renewcommand{\CalibStepCal}{2.6}
\providecommand{\CalibStepPrice}{}\renewcommand{\CalibStepPrice}{10.3}
\providecommand{\CalibSpendPct}{}\renewcommand{\CalibSpendPct}{10.6}

%% file: sections/tables/fig6_slots.tex
\providecommand{\DoseHeroMin}{}\renewcommand{\DoseHeroMin}{66.5}
\providecommand{\DoseHeroMax}{}\renewcommand{\DoseHeroMax}{69.7}
\providecommand{\DoseLitPeak}{}\renewcommand{\DoseLitPeak}{68.3}
\providecommand{\DoseLitFloor}{}\renewcommand{\DoseLitFloor}{46.7}
\providecommand{\GapSlopeLo}{}\renewcommand{\GapSlopeLo}{2.3}
\providecommand{\GapSlopeMid}{}\renewcommand{\GapSlopeMid}{7.4}

\providecommand{\GapSlopeHi}{}\renewcommand{\GapSlopeHi}{12.9}
\providecommand{\BandSpendHeroLo}{}\renewcommand{\BandSpendHeroLo}{7.4}
\providecommand{\BandSpendHeroHi}{}\renewcommand{\BandSpendHeroHi}{12.7}
\providecommand{\BandSpendEagerLo}{}\renewcommand{\BandSpendEagerLo}{13.3}
\providecommand{\BandSpendEagerHi}{}\renewcommand{\BandSpendEagerHi}{11.5}

%% file: sections/tables/app_slots.tex
\providecommand{\ApRosterAtoms}{}\renewcommand{\ApRosterAtoms}{25}
\providecommand{\ApNsrHatLooseActive}{}\renewcommand{\ApNsrHatLooseActive}{81}
\providecommand{\ApNsrZeroTightCap}{}\renewcommand{\ApNsrZeroTightCap}{70}
\providecommand{\ApStaleHatB}{}\renewcommand{\ApStaleHatB}{12.2\stdv{1.0}}

\providecommand{\ApLifeboatRho}{}\renewcommand{\ApLifeboatRho}{+1.00}
\providecommand{\ApLifeboatMaxGap}{}\renewcommand{\ApLifeboatMaxGap}{$+11.0$\,pp [+10.4, +11.7]}
\providecommand{\ApLifeboatNB}{}\renewcommand{\ApLifeboatNB}{-0.01}
\providecommand{\ApLifeboatSat}{}\renewcommand{\ApLifeboatSat}{+2.01}
\providecommand{\ApAwLitPeak}{}\renewcommand{\ApAwLitPeak}{68.3}
\providecommand{\ApAwLitFloor}{}\renewcommand{\ApAwLitFloor}{46.7}
\providecommand{\ApAwHatMin}{}\renewcommand{\ApAwHatMin}{66.5}
\providecommand{\ApAwHatMax}{}\renewcommand{\ApAwHatMax}{69.7}
\providecommand{\ApCorrExcess}{}\renewcommand{\ApCorrExcess}{1.94\stdv{0.81}}
\providecommand{\ApFourBGap}{}\renewcommand{\ApFourBGap}{$+2.1$\,pp [+1.5, +2.6]}
\providecommand{\ApClosureOff}{}\renewcommand{\ApClosureOff}{0.00}
\providecommand{\ApClosureOn}{}\renewcommand{\ApClosureOn}{2.22}
\providecommand{\ApGapAtFive}{}\renewcommand{\ApGapAtFive}{6.2}
\providecommand{\ApGapRatio}{}\renewcommand{\ApGapRatio}{1.6}
\providecommand{\ApNumLadder}{}\renewcommand{\ApNumLadder}{0.060 < 0.175 < 0.236}
\providecommand{\ApStaleCost}{}\renewcommand{\ApStaleCost}{79.5}
\providecommand{\ApFrontierGap}{}\renewcommand{\ApFrontierGap}{$+10.0$\,pp [+8.6, +11.4]}
\providecommand{\ApFrontierGapM}{}\renewcommand{\ApFrontierGapM}{10.0}

\providecommand{\ApCondMatched}{}\renewcommand{\ApCondMatched}{$+4.9$\,pp [+3.3, +6.6]}
\providecommand{\ApCondHalf}{}\renewcommand{\ApCondHalf}{1.66}
\providecommand{\ApCondEight}{}\renewcommand{\ApCondEight}{$+4.2$\,pp [+2.5, +6.0]}
\providecommand{\ApPairsRealized}{}\renewcommand{\ApPairsRealized}{1016}
\providecommand{\ApGapPriceDiff}{}\renewcommand{\ApGapPriceDiff}{$+3.8$\,pp [+1.8, +5.8]}
\providecommand{\ApDordMid}{}\renewcommand{\ApDordMid}{+3.09}
\providecommand{\ApDordHigh}{}\renewcommand{\ApDordHigh}{+7.28}
\providecommand{\ApCalibSegOne}{}\renewcommand{\ApCalibSegOne}{$+3.3$\,pp [-4.1, +10.8]}
\providecommand{\ApCalibSegTwo}{}\renewcommand{\ApCalibSegTwo}{$+9.6$\,pp [+5.8, +13.4]}

\providecommand{\ApRevive}{}\renewcommand{\ApRevive}{$+5.0$\,pp [+1.4, +8.6]}
\providecommand{\ApAbDeepest}{}\renewcommand{\ApAbDeepest}{-13.0}
\providecommand{\ApAlWhitTight}{}\renewcommand{\ApAlWhitTight}{-2.2}
\providecommand{\ApAlMyoTight}{}\renewcommand{\ApAlMyoTight}{-4.8}
\providecommand{\ApAlWhitLoose}{}\renewcommand{\ApAlWhitLoose}{-1.9}
\providecommand{\ApDomSpendSix}{}\renewcommand{\ApDomSpendSix}{+0.03}
\providecommand{\ApDomCondLoose}{}\renewcommand{\ApDomCondLoose}{-0.17}
\providecommand{\ApStaleEagerB}{}\renewcommand{\ApStaleEagerB}{21.9\stdv{1.0}}
\providecommand{\ApStaleOracleB}{}\renewcommand{\ApStaleOracleB}{3.4\stdv{1.3}}
\providecommand{\ApStaleBigmemB}{}\renewcommand{\ApStaleBigmemB}{53.1\stdv{1.8}}
\providecommand{\ApClassTop}{}\renewcommand{\ApClassTop}{+8.4}
\providecommand{\ApClassFloor}{}\renewcommand{\ApClassFloor}{+3.2}
\providecommand{\ApHatFive}{}\renewcommand{\ApHatFive}{73.4}
\providecommand{\ApZeroFive}{}\renewcommand{\ApZeroFive}{67.2}

%% file: sections/0_abstract.tex
\begin{abstract}
Deployed agents run on handed-over knowledge: a frozen policy consults a
briefing of consolidated items written before the stream begins. The world
then moves while the store stands still: paths close, flags change, price
bands move; every item was true at handover, and the failure is staleness,
not ignorance. We introduce \sys{}, which treats \emph{revalidation as a
priced errand}: a recheck competes with the task it protects for the same
scarce actions, funded only when the value per action of resolving a
doubt clears a running wage. The errand index is single-peaked,
vanishing at both ends of belief, so certainty in either direction costs
nothing; free en-route receipts maintain on-path knowledge, and repair
writes a version, never a deletion. Under equal action budgets in two drifting
tool-use worlds, \sys{} clears every non-oracle policy on the
preregistered calibers, primary in every setting and conditional at every
binding budget, leading eager revalidation by \FrontierGap{}pp at the
base cap. Restraint wins: given no cap, \sys{} stops on its own,
spending \SelfCapSpend\% of steps, while uncapped eager revalidation
spends \EagerNocapSpend\% and still finishes \EagerNocapTrail{}pp behind
capped \sys{}. The
margin sits where the briefing's coverage is thinnest, the shadow price of
long-tail knowledge: a small budget, well priced, beats a bigger store
that never rechecks.
\end{abstract}

%% file: sections/1_intro.tex
\section{Introduction}

Deployed agents increasingly run on consolidated knowledge: a frozen policy
consults a briefing of distilled experience (runbooks, configuration facts,
tool notes) and reuses it across tasks \citep{Shinn2023NEURIPS,
Park2023UIST, Wang2024TMLR, Xu2025NEURIPS, Fang2025ARXIV}. A maturing line of work governs
what such a store admits and keeps, pricing capacity by the byte and
entries by their net value \citep{Wu2026ARXIVa, Wu2026ARXIVb, Mi2026ICML, Wu2026TNSE, Wu2026ARXIVc}.
It is widely assumed that an entry that survives admission and curation is
an asset from then on: validity is settled at write time, and upkeep during
deployment has no budget line. In this work we take a critical look at that
write-time contract on streams where the world moves and nothing announces
it: paths close, flags change, price bands move. The contract fails
quietly: every item can be true at handover and the store still decays,
because what an item asserts outlives the moment it was checked. When actions are scarce, which of the beliefs an agent already
trusts is worth an action to recheck, and when?

\looseness=-1
A first family of responses governs the store itself: admission gates at
write or induction time \citep{Shu2026ARXIV, Mi2026ICML}, transactional
commit with typed repair \citep{Li2026ARXIV}, and belief revision over
versioned memory graphs \citep{Park2026ARXIV}. A second family goes after
freshness directly: models audit their own knowledge \citep{Guo2026ARXIV},
dissonance triggers re-probing of the environment \citep{Yin2026ARXIV},
and stale facts are located and rewritten in the weights
\citep{Meng2022NEURIPS} or refreshed against live search \citep{Vu2024ACL}. Each is, however, half an answer: (i)~\emph{store-side}
governance prices size, value, and consistency, so an entry's validity is
fixed at admission and revised only when contradiction arrives; freshness
is an event the store waits for, not a quantity it manages. And
(ii)~\emph{freshness-side} checks are unpriced: verdicts a model issues
about itself over-validate \citep{Andrade2026ICLR}, detection scores fail
to become better behavior \citep{Chao2026ARXIV}, and the checks that do act
in the world fire immediately and undifferentiated, blind to what resolving
one doubt is worth against the task step it displaces. Nor do benchmarks force the issue:
conversational suites rarely contain the contradictions that would
separate methods \citep{Singh2026ARXIV}, and test-time continual learning
suites record the symptom while their answer keys anchor to design-time
state \citep{Zhang2026ARXIV}.

Our answer is a single principle: \emph{revalidation is a priced errand}. A
recheck is an action in the same world the tasks run in; it competes with
the task it protects for the same scarce steps, and it deserves funding
exactly when the value of resolving a doubt, per action, clears the going
wage of the budget. The fulcrum is that certainty is free at both ends: an
item almost surely fresh needs no check and an item almost surely stale
needs a takedown, so the value of resolving is single-peaked and every
funded action goes to decision-relevant doubt. \sys{} instantiates the
principle: drift-aware beliefs over the atoms beneath each item, moved only
by evidence about the world; free en-route receipts collected wherever task
traffic already goes; a deadband gate that funds the cheapest sufficient
errand once its value per action clears that wage; and a versioned
lifecycle that suspends and supersedes instead of deleting.

Our work makes three contributions. (1)~We formulate budgeted maintenance
of consolidated memory and give it closed forms: belief relaxation under
two-state drift, the single-peaked value of resolving, a deadband whose
width is exact in the single-atom regime, and a durability horizon for how
long a purchased answer stays decision-relevant. A preregistered prediction
test backs the forms: across four flip groups, every seed-level measured
gain lands on or above its closed-form forecast, so the model reads as a
conservative bound. (2)~We build \sys{}, a scheduler that collects free
receipts before pricing anything, buys the cheapest sufficient grade of
check, repairs by versioning, and schedules in closed form without spending
a model call. Priced against seven policies on one substrate, it clears
every non-oracle policy on the preregistered calibers, the primary in
every setting and the conditional at every binding budget,
and a five-policy ladder at matched spend splits the margin into a
calibration step and a pricing step.
(3)~We map when pricing matters: across a tenfold drift sweep, half-life
priors collapse once real drift decouples from the calendar while \sys{}
holds flat, and by coverage-gap severity the premium climbs monotonically
as the coverage under a step thins.

\looseness=-1
Two findings run against intuition. First, restraint wins: given no cap at
all, \sys{} stops on its own at \SelfCapSpend\% of steps, while uncapped
eager revalidation spends \EagerNocapSpend\%, \SpendMultiple$\times$ more,
and still finishes \EagerNocapTrail{}pp behind the capped system; the wage
governs before the cap does. Second, the premium is a shadow price: at
equal budget the margin over eager revalidation climbs from \GapSlopeLo{}pp
on fully covered steps to \GapSlopeHi{}pp where the briefing's coverage is
thinnest, and the headline \FrontierGap{}pp is that curve's
density-weighted average, so paid ordering matters most exactly where the
briefing runs out. What a maintenance budget buys is set by its prices
before its size.

%% file: sections/2_related.tex
\section{Related Work}

\paragraph{Governing a growing store.}
\looseness=-1
Deployed agents write their experience into persistent stores and reuse it
across tasks \citep{Shinn2023NEURIPS, Park2023UIST, Wang2024TMLR,
Packer2023ARXIV, Wang2025ICML, Ouyang2026ICLR, Xu2025NEURIPS,
Chhikara2025ARXIV, Wu2026ARXIVd, Ding2026TAAS, Wu2026ICDCS}. Governance of these stores has matured along two budgets:
CrystalMem schedules capacity and fidelity under a byte budget, and
on-device curation scores entries by net value per byte
\citep{Wu2026ARXIVa, Wu2026ARXIVb, Wu2026IoTJ, Tchalla2026ARXIV}. Lifecycle and repair interfaces are
appearing: transactional commit with typed cascade repair \citep{Li2026ARXIV},
belief-revision semantics over versioned memory graphs \citep{Park2026ARXIV},
and admission gates at write or induction time \citep{Shu2026ARXIV,
Mi2026ICML}. Across this line an entry's validity is fixed when it is written,
or revised when new input happens to contradict it. GLOVE comes closest to
maintenance: it detects dissonance and re-probes the environment, but probing
is immediate, undifferentiated, and unbudgeted \citep{Yin2026ARXIV}. The
store's size, value, and consistency are all priced; whether its contents
still hold in the world, and which check is worth an action today, is not.

\paragraph{Free verdicts and where they fail.}
\looseness=-1
A natural fix is to let the model audit its own knowledge, and the evidence
is against it: self-verification over-validates and stays so under test-time
scaling \citep{Andrade2026ICLR, Guo2026ARXIV, Jiang2025ARXIV}.
Static contract checks catch only enumerable preconditions
\citep{Fan2026ARXIV}; confidence read off historical success inflates on
exactly the entries reused most \citep{Cui2026ICMLW, Asadolahi2026ARXIV}; and
read-time filtering saturates: even oracle staleness labels close little of
the remaining gap \citep{Sun2026ARXIV}, detection scores fail to become
better behavior \citep{Chao2026ARXIV}, and standard conversational benchmarks
rarely contain the contradictions that would separate methods
\citep{Singh2026ARXIV}. Benchmarks built for test-time continual learning
record the symptoms, agents losing track of moved objects and outdated rules,
while their answer keys anchor to design-time state
\citep{Zhang2026ARXIV}. These verdicts share one property: they are
free. A verdict about a world that moved costs an action in that world.

\paragraph{Pricing freshness.}
\looseness=-1
Keeping a mirror fresh under a budget is classical: crawl scheduling under
known change rates \citep{Cho2000SIGMOD, Cho2003TODS, Wu2025TON} and noisy change
signals \citep{BusaFekete2025ARXIV, Fang2025TON, Fang2025JSAC}, age-of-information updating
\citep{Kaul2011SECON, Kaul2012INFOCOM, Wu2023ACCESS, Fang2022JSAC, Wu2026MNET, Wu2026COMST, Ding2026ICDCS}, index policies that price activation
in changing worlds \citep{Gittins1979JRSSB, Whittle1988JAP, NinoMora2026ARXIV, Ding2025IPCCC, Ding2026ICNC, Xing2026ACR, Dong2026TCCN},
value of information as the currency \citep{Howard1966TSSC, Wu2023MPE}, event-triggered
rules for when a costly check fires \citep{Wang2026ECMLPKDD, Fang2026ARXIV}, and
re-audit budgets on live tool registries \citep{Bharti2026ARXIV, Ding2026MNET, Ding2026ARXIVb, Pudasaini2026HPSR, Ding2026ARXIVa}. In
these problems freshness is itself the objective, every check is a dedicated
probe of uniform kind, and the answer feeds a mirror rather than a decision
loop. \sys moves this economics inside the agent's own action budget: a
recheck competes with the task it serves, items carry applicability
conditions and provenance, the value of resolving is single-peaked so
certainty at either end spends nothing, and repair is versioned for regimes
that return.

%% file: sections/3_method.tex
\section{Revalidation as a Priced Errand}
\label{sec:method}

\begin{figure}[t]
  \centering
  \includegraphics[width=0.95\linewidth]{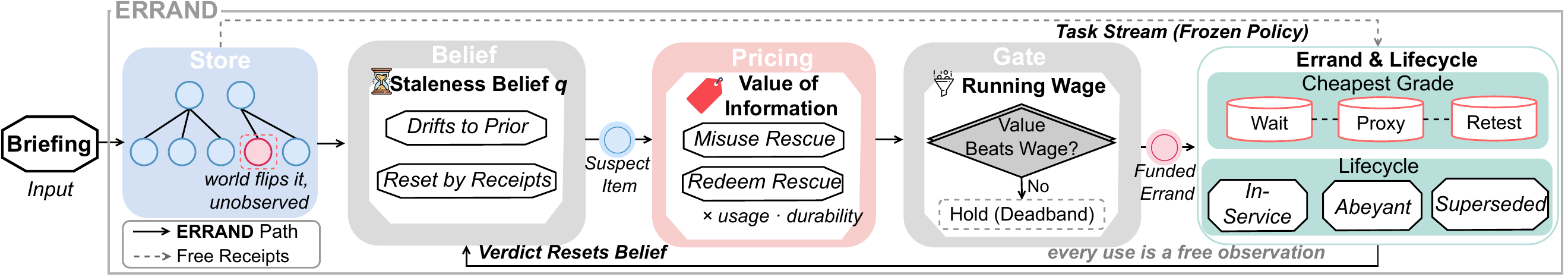}
  \caption{\textbf{The budget buys trajectory changes, never attention.}
  Beliefs over grounded atoms drift toward the prior and reset on free
  en-route receipts; the single-peaked index prices each off-path item; the
  deadband gate funds an errand only when value per action clears the running
  wage $\wage$; repair chooses the cheapest sufficient grade and writes a
  version, never a deletion. Abeyance closes use while the comparator keeps
  watching.}
  \label{fig:loop}
\end{figure}

\looseness=-1
\sys runs on any deployed agent whose frozen policy consults a store of
consolidated knowledge at retrieval time. It changes neither the policy nor
what the store contains at rest; it governs a single scarce quantity: the
actions the agent diverts from tasks to recheck what it already believes.
Four routines implement one loop (Figure~\ref{fig:loop}): a drift-aware belief over every grounded
atom, a single-peaked index that prices what resolving each item is currently
worth, a deadband gate that funds an errand only when that price clears a
running wage, and a versioned lifecycle that suspends and supersedes instead
of deleting.

\subsection{Setting and Threat Model}

\looseness=-1
A frozen-policy agent serves a task stream over a horizon of $T$ steps from a
briefing $M$: consolidated knowledge items handed over at deployment
(runbooks, configuration facts, tool notes) rather than earned through the
agent's own trajectories. Each item $i$ grounds in a set $A_i$ of
\emph{atoms}, individually checkable facts about the world: a path exists, a
flag is accepted, a price band holds. The world moves on its own schedule.
Atom $a$ flips between its recorded value and its negation as a two-state
Markov chain with flip-out intensity $\lambda_a$ and flip-back intensity
$\mu_a > 0$; regimes leave and return. Absent observation, the agent's
staleness belief $q_a$, the probability that $a$ no longer holds as recorded,
relaxes in closed form,
{\small
\begin{equation}
    \label{eq:relax}
    q_a(t + \Delta) \;=\; q_a^{\infty} \;+\;
    \bigl(q_a(t) - q_a^{\infty}\bigr)\, e^{-\kappa_a \Delta},
\end{equation}
}where $\kappa_a = \lambda_a + \mu_a$ is the total flip intensity and
$q_a^{\infty} = \lambda_a / \kappa_a$ the stationary suspicion. An item is
suspect once any of its atoms is, $q_i = 1 - \prod_{a \in A_i} (1 - q_a)$
under declared independence; instrumentation keeps most items single-atom,
the regime where every closed form below is exact, and a preregistered
correlated-flip suite stresses the declaration deliberately.

\looseness=-1
Maintenance competes with work. At each step the agent either advances its
task or runs an \emph{errand}: a targeted recheck that resolves an atom by
acting on the environment, at cost $c$ actions, under a cap of $b$ errand
actions per hundred steps. Beliefs are not bookkeeping over past success
rates; they move only on evidence about the world, and evidence has exactly
two sources. Whenever a task step happens to traverse a site, the agent reads
the outcome as an \emph{en-route receipt}, free corroboration or refutation
of every belief it touches. Everything else must be bought. The threat needs
no adversary: knowledge off the task path collects no receipts, its belief
relaxes toward $q^{\infty}$, and the next task routed there executes against
a world that moved. The failure is staleness, not ignorance; every item was
true when handed over. We assume atoms are cheap to check once reached,
relative to judging an item's usefulness in the abstract; our environments
make the reaching, not the checking, the scarce step.

\subsection{The Errand Index and the Deadband Gate}

\paragraph{En-route receipts.}
Every use is an observation. A served item that executes against the world
returns a receipt at zero marginal cost, and the scheduler collects all of
them before it prices anything. High-traffic knowledge therefore maintains
itself; the budget exists for the tail that traffic does not reach.

\paragraph{The single-peaked value of resolving.}
\looseness=-1
What is resolving one item worth? If the item is in service, learning it
flipped saves the misuse loss $\ell_i$, and the chance the answer is
``flipped'' is $q_i$. If the item sits in abeyance, learning it held redeems
the forgone gain $g_i$, with chance $1 - q_i$. The value of information is
whichever rescue is live,
{\small
\begin{equation}
    \label{eq:voi}
    \mathrm{VOI}_i(q_i) \;=\; \min\bigl\{\, q_i\, \ell_i,\;\;
    (1 - q_i)\, g_i \,\bigr\},
\end{equation}
}and the shape does the governing: $\mathrm{VOI}$ vanishes at both ends,
because an item almost surely fresh needs no check and an item almost surely
stale needs a takedown, not a test. Certainty in either direction is free. A
candidate errand at site $p$ scales this peak by how often the answer matters
and for how long it stays decision-relevant,
$V(i,p,t) = \mathrm{VOI}_i(q_i)\; \hat{u}_i \min\{\tau_i,\, T - t\}\;
\xi(p,i)$, with $\hat{u}_i$ the item's usage rate, $\tau_i$ its durability
horizon (below), and $\xi(p,i) \in (0,1]$ a locality discount. No static
checklist reproduces this ordering: the index is a live function of belief,
traffic, and remaining time rather than a fixed property of the item.

\paragraph{One comparison, two readings.}
Each step the gate compares candidates' value per action, $V/c$, to a wage
$\wage$, the shadow price of the errand budget, in the lineage of
restless-bandit index policies \citep{Whittle1988JAP, Huang2025TMC, Wu2025WASA}. Read across
candidates, the comparison ranks: fund the highest $V/c$ first, extending to
co-located batches greedily while each marginal member clears the wage. Read
across time, the comparison stops: on steps where no candidate clears it,
the agent works, and the store sits in its \emph{deadband}. Inaction is a
priced decision, held accountable by a resolve-by date and a minimum revisit
gap $\Delta_{\min}$ \citep{Wang2026ECMLPKDD, Wu2025RACS, Fang2026GLOBECOM, Duan2023COMST, Pan2023SCIS}, so deferral waits for the free
channel rather than drifting past it. In the single-atom regime the band has
closed width,
{\small
\begin{equation}
    \label{eq:band}
    \lvert \mathcal{D}_i \rvert \;=\;
    \frac{\wage\, c\, (g_i + \ell_i)}{g_i\, \ell_i\, N_i},
\end{equation}
}where $N_i$ counts the item's usage receipts: costlier checks widen the
band, heavier use narrows it, so doubt is tolerated longest exactly where it
is cheapest to hold.

\paragraph{Three grades of check.}
\looseness=-1
Errands come in grades, cheapest sufficient grade first: an en-route deferral
(wait for traffic already headed there), a proxy reading (an adjacent
observable that moves with the atom), and a direct retest (reach the site and
act on it). Grade changes only $c$ and the fraction of doubt one visit
resolves; the same wage prices all three, so the gate never pays retest
prices for proxy-grade questions.

\subsection{A Lifecycle That Closes Use, Never Eyes}

\looseness=-1
Repair is a state change, not a deletion. An item runs \emph{in-service},
where retrieval presents it and every use pays receipts, or \emph{abeyant},
where retrieval withholds it and its belief relaxes toward the prior; a
resolved flip \emph{supersedes} the recorded value in place, versioned so a
regime that returns costs a lookup rather than relearning. One rule is
load-bearing: abeyance closes use, never eyes. Suspending an item removes it
from execution, so its receipts stop because its uses stop; but the
comparator that matches en-route observations against recorded values keeps
reading every traversal, gated by nothing. It must, because gating reads
would starve the redemption evidence: an abeyant item watched for free either
revalidates on the next traversal, at zero budget, or stays suspect and
competes for the same wage through the $g$-branch of its $\mathrm{VOI}$.
Detection and revival are one quantity in one currency.

\looseness=-1
How long a purchased answer stays decision-relevant has its own closed form.
After a resolution the belief restarts near certainty and relaxes toward
$q_i^{\infty}$; with a decision threshold $q_i^{\ast}$, the belief at which
the item's disposition flips, the durability horizon is the re-crossing
time,
{\small
\begin{equation}
    \label{eq:tau}
    \tau_i \;=\; \frac{1}{\kappa_i}\,
    \ln \frac{\lvert q_i^{\infty} - q_i \rvert}
             {\lvert q_i^{\infty} - q_i^{\ast} \rvert},
\end{equation}
}taken as $\tau_i = \infty$ when $q_i^{\ast}$ does not lie between the
running belief and $q_i^{\infty}$: the answer never expires by drift alone.
These durable purchases are what the index buys first, since $\tau_i$
multiplies $V$; the scheduler prefers questions whose answers stay answered.

\subsection{The Wage and the Meter}

\looseness=-1
The wage is the budget speaking. $\wage$ is maintained as a dual estimate: it
rises while the cap binds and decays while slack, so under a tight cap it
reports scarcity and under a loose one it settles at an intrinsic floor, the
price below which no errand pays even with budget left over. Spending
therefore self-limits before the cap does. Everything the scheduler computes
is closed-form arithmetic over the $k$ retrieved candidates, $O(k)$ per step;
scheduling never spends a model call. The scheduler also explains inaction:
every no-spend step logs which clause held (no candidate, index at zero, gate
below wage, cap reached, check in flight), so the deadband is an auditable
object rather than a silence, and the calibration of $q$ itself is a measured
quantity: an agent that orders its doubt better converts the same budget into
more repair.

\begin{summarybox}
\noindent\textbf{The \sys contract.} (i)~Certainty is free at both ends: the
single-peaked index sends every funded action to decision-relevant doubt, and
in the single-atom regime the deadband width and the durability horizon are
exact, so inaction is a computed, closed-form region.
(ii)~The budget buys trajectory, never attention: en-route comparison is
never gated, suspension closes use only, and repair is versioned, so
redemption and regime return cost a traversal and a lookup.
\end{summarybox}

%% file: sections/4_experiments.tex
\section{Experiments}
\label{sec:experiments}

\setlength{\textfloatsep}{6pt plus 2pt minus 1pt}
\setlength{\intextsep}{5pt plus 2pt minus 1pt}
\setlength{\floatsep}{6pt plus 2pt minus 1pt}
\makeatletter
\setlength{\@neuripsabovecaptionskip}{3\p@}
\makeatother
\setlength{\abovecaptionskip}{3pt}

\subsection{Setup}

\paragraph{Environment and drift.}
\looseness=-1
Every run deploys a frozen Qwen3-8B agent for $T{=}2{,}000$ steps on a
scripted tool-use world of sites, paths, and flags, with the first $300$
steps a warm-up excluded from every metric. The briefing hands over $26$
consolidated items grounded in atoms the environment actually flips: paths
close, flags change, price bands move, under four preregistered intensity
schedules, and nothing in the stream announces a flip. Tasks arrive as
dispatch orders; the dispatch mix sets the coverage gap, the share of
knowledge-bearing steps routed off the briefing's high-traffic sites, in a
base and a high regime. A second tool-API world drives the drift-rate sweep
of Figure~\ref{fig:mech}c, and a 4B backbone and a $500$-step stream stress
the capability floor and the short horizon
(full configuration: Appendix~\ref{app:details}).

\paragraph{Protocol and metrics.}
\looseness=-1
We preregistered two calibers before any full wave: ITT, success over all
errand-relevant steps, and conditional success on knowledge-bearing steps,
which isolates what the store contributed. ITT is primary throughout:
maintenance raises it both by cleaner serves and by stale serves averted,
while the conditional caliber scores only the serves that happen, on each
arm's realized judged domain (reported per arm in
Appendix~\ref{app:results}). Table~\ref{tab:main} reports both. Budgets
are hard caps of $b$ errand
actions per hundred steps, and equal-budget comparisons match measured
spend, not nominal allowance. The frontier arms run five seeds per cap
and eleven on the high-gap tier; baseline arms and breadth settings run
three; error bars are seed s.d. Oracle arms read ground-truth staleness labels
and enter only as priced ceilings. The preregistered analysis plan, with
every band and outcome, is Appendix~\ref{app:power}.

\paragraph{Implementation.}
\sys{} runs its preregistered defaults for the gate, the three errand
grades, and the dual wage estimate, frozen before the full waves
(pseudocode and constants: Appendix~\ref{app:algorithms}; derivations:
Appendix~\ref{app:forms}). Scheduling is closed-form arithmetic over the retrieved
candidates; it never spends a model call (Table~\ref{tab:overhead}).

\subsection{Priced Revalidation Wins at Every Budget}

\begin{wrapfigure}{r}{0.62\textwidth}
  \vspace{-1\baselineskip}
  \centering
  \begin{subfigure}[t]{0.49\linewidth}
    \includegraphics[width=\linewidth]{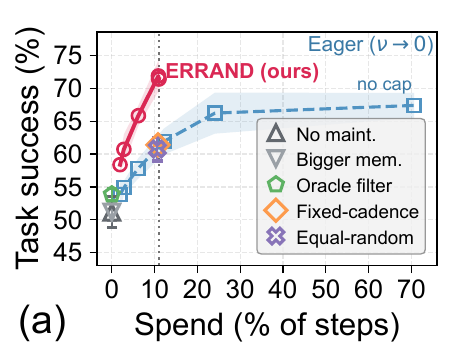}
  \end{subfigure}\hfill
  \begin{subfigure}[t]{0.49\linewidth}
    \includegraphics[width=\linewidth]{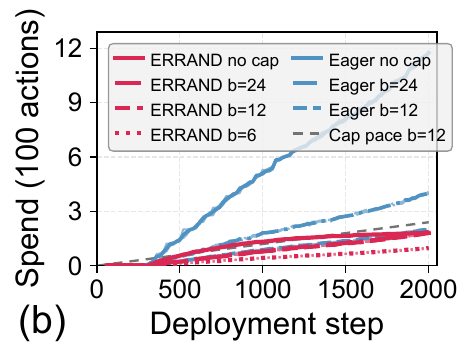}
  \end{subfigure}
  \caption{\textbf{Priced errands buy the frontier, and the scheduler
  stops spending on its own.} \textbf{(a)} Success vs.\ measured spend
  across caps; the dotted line is where uncapped \sys{} settles.
  \textbf{(b)} Every \sys{} cap converges to one pace.
  Bands/whiskers: seed s.d.}
  \label{fig:frontier}
  \vspace{-1\baselineskip}
\end{wrapfigure}
Left alone, consolidated knowledge does not age gracefully: the
no-maintenance arm succeeds on \StaleFresh\% of knowledge steps while its
items still hold and \StaleStale\% once they have flipped, a \StaleCost{}pp
collapse on the same stream. The question is not whether to spend
maintenance actions but what a fixed allowance buys. Figure~\ref{fig:frontier} answers on the spend axis: at
every measured spend level \sys{} sits above eager revalidation, by
\FrontierBTwoGain{}pp at the tightest cap and \FrontierGap{}pp at the base
cap, while the zero-spend family is pinned to a \NuInfFloor\% floor.

\paragraph{Baselines.}
\looseness=-1
Table~\ref{tab:main} prices seven policies on one substrate:
(1)~\textit{zero-spend anchors}: no maintenance, and a bigger store with no
maintenance; (2)~\textit{unpriced spenders} that pay the same budget without
prices: eager revalidation ($\nu{\to}0$), fixed-cadence rechecks, and
equal-budget random errands; (3)~an \textit{oracle detect-and-filter} that
reads ground-truth staleness labels, italicized as a ceiling; and
(4)~\sys{}. Zero-spend rows are cap-invariant, so one run repeats across
budget columns; Appendix~\ref{app:results} extends the grid.

\paragraph{Key insights.}
\looseness=-1
\sys{} tops ITT in every setting and the conditional caliber at every
binding budget: \HeroBaseItt\% at the base
cap against \EagerBaseItt\% for eager revalidation, margins of
\BudgetGapTight/\BudgetGapBase/\BudgetGapLoose{}pp across caps $b{=}6/12/24$
(Figure~\ref{fig:money}c), and \TDeltaZero{}pp on the five-setting average
of Table~\ref{tab:main}. The unpriced spenders are not a ladder toward
\sys{}; they cluster within about a point of one another, and the regime
shift does not close the gap: moving from the base to the high coverage-gap
regime, the best baseline gives up \BBDropItt{}pp while the ITT margin holds
at \FrontierGapHigh{}pp and the conditional margin widens from
\CondGapBase{} to \CondGapHigh{}pp (Figure~\ref{fig:money}ab). The oracle
shows why deletion is not maintenance: reading labels no deployed system
has, it wins the conditional caliber yet trails every spending policy on
ITT, buying precision with coverage. Notably, the scheduler prices itself
out of overspending. Given no cap at all, \sys{} stops at \SelfCapSpend\%
of steps; uncapped eager revalidation spends \EagerNocapSpend\%,
\SpendMultiple$\times$ more, and still finishes \EagerNocapTrail{}pp behind
capped \sys{} (Figure~\ref{fig:frontier}). The wage, not the cap, is doing
the governing. The margin also compounds: at a $500$-step prefix the
frontier gap already reads \ApGapAtFive{}pp and reaches \ApFrontierGapM{}pp
over the full horizon, a ratio of \ApGapRatio{}; every item kept fresh
keeps paying on every later use (Appendix~\ref{app:res-horizon}).

\input{sections/tables/tab1_main}
\begin{figure}[t]
\centering
\includegraphics[width=\linewidth]{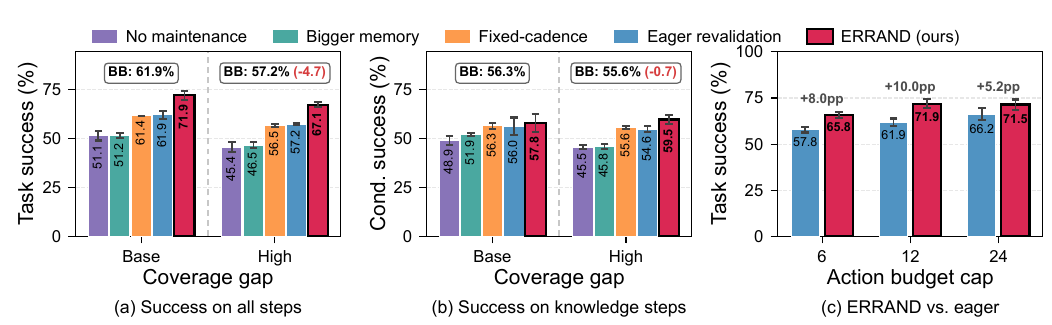}
\caption{\textbf{Priced revalidation wins on both calibers, and the ITT
margin survives every budget cap.} Success on \textbf{(a)} all steps and
\textbf{(b)} knowledge steps under both coverage-gap regimes; boxes: best
baseline (BB) and its cross-regime drop. \textbf{(c)} \sys{}
vs.\ eager revalidation across caps. Error bars: seed s.d.}
\label{fig:money}
\end{figure}

\subsection{The Scheduler's Economics Are Real}

\begin{figure}[t]
  \centering
  \includegraphics[width=0.245\linewidth]{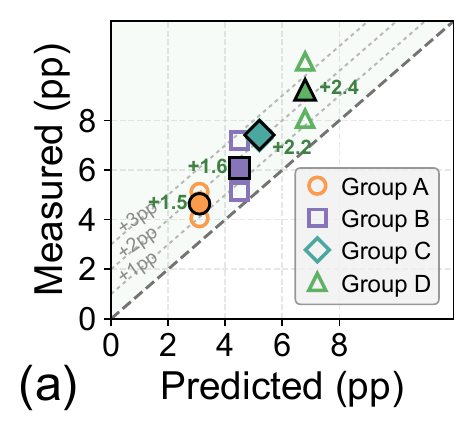}\hfill
  \includegraphics[width=0.245\linewidth]{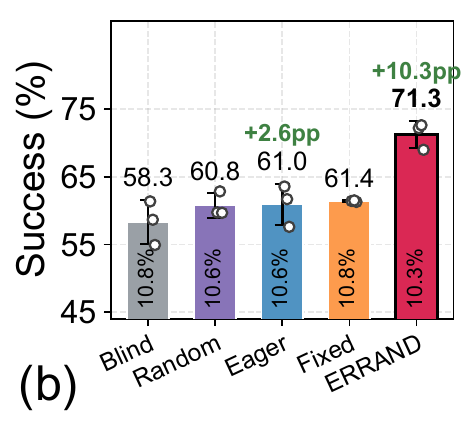}\hfill
  \includegraphics[width=0.245\linewidth]{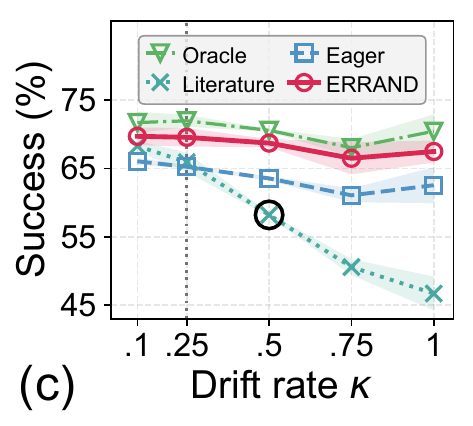}\hfill
  \includegraphics[width=0.245\linewidth]{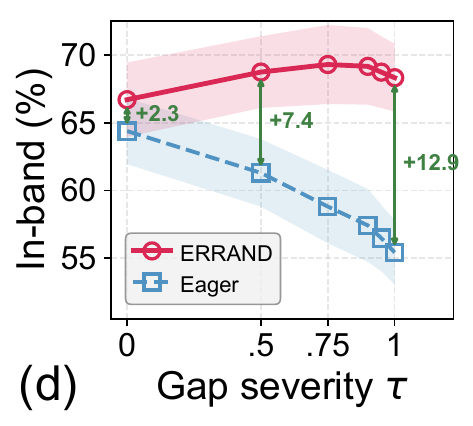}
  \caption{\textbf{The prices are calibrated, and they track the world, not
  the calendar.} \textbf{(a)} All twelve points clear the diagonal.
  \textbf{(b)} Five orderings at matched spend; dots: seeds. \textbf{(c)}
  \sys{} holds flat; past $\kappa^{\ast}$ (dotted) the half-life policy
  collapses (circled). \textbf{(d)} By coverage-gap severity band, \sys{}
  holds while eager falls away; arrows: the conditional gap in points at
  equal budget. Bands: seed s.d.}
  \label{fig:mech}
\end{figure}
\looseness=-1
Why does the same budget buy more when priced? Section~\ref{sec:method}
claims the index knows what resolving each item is worth;
Figure~\ref{fig:mech}ab tests that claim twice. First, prediction. For four
preregistered flip groups we computed the coverage gain the closed forms
predict, then measured it: all twelve seed-level points land on or above
the diagonal, every flip group clears its own forecast on average, and the
measured gain runs \CorrflipExcessLo{} to \CorrflipExcessHi{}pp above the
prediction. The excess has a mechanical reading: en-route receipts keep
repairing items the closed forms book as unresolved, so the model
undercounts what maintenance recovers. It is a conservative bound, not a
fit. Second, allocation. Five policies spend the same \CalibSpendPct\% of
steps (Figure~\ref{fig:mech}b) and differ only in what they believe and how
they order: random, eager, and fixed-cadence errands land within a point of
one another, calibrating beliefs alone buys \CalibStepCal{}pp over a
blind-belief scheduler, and pricing errands on those beliefs buys another
\CalibStepPrice{}pp. The ladder is the method disassembled: the blind arm
keeps the gate and the caps but scores items with an uncalibrated belief;
random keeps the spend and drops the ordering; fixed-cadence replaces doubt
with a clock. The two rungs isolate the machinery: calibration fixes
what the scheduler believes; pricing fixes which doubt is worth an action.
At matched spend, the entire frontier gap is an allocation gap.

\subsection{The Premium Is a Shadow Price}

\looseness=-1
When does pricing matter? Figure~\ref{fig:mech}cd maps the two regime axes.
On the drift axis, \sys{} holds between \DoseHeroMin{} and \DoseHeroMax\%
ITT across a tenfold range of flip intensity, tracking the known-rate
reference. The literature policy, cadences set from published half-life
priors and the field's default answer to staleness, starts competitive at
\DoseLitPeak\% and collapses to \DoseLitFloor\% once real drift decouples
from its priors: it keeps refreshing what the calendar flags, not what the
world flipped. Notably, the collapse point, not the average, is the
argument against fixed schedules; below $\kappa^{\ast}$ the two are
indistinguishable.

\looseness=-1
On the coverage axis, the premium is conditional, and it sits exactly
where coverage is thin. Rank knowledge-bearing steps by coverage-gap
severity, the share of required atoms the briefing does not cover, and
read each band separately (Figure~\ref{fig:mech}d): \sys{} holds its
in-band success nearly flat while eager revalidation falls away, so the
conditional gap widens from \GapSlopeLo{}pp on fully covered steps,
through \GapSlopeMid{}pp at the median band, to \GapSlopeHi{}pp in the
binding tail; the unconditional margin of \FrontierGap{}pp is this curve's
density-weighted average. The budget follows the same gradient: at equal
budget, \sys{} raises its maintenance rate from \BandSpendHeroLo\% on the
safest band to \BandSpendHeroHi\% in the tail, while eager revalidation
runs the opposite slope (\BandSpendEagerLo\% down to \BandSpendEagerHi\%),
paying most where receipts are already free. Paid ordering is the shadow
price of long-tail knowledge. The two axes give the deployment rule: on
covered traffic and below $\kappa^{\ast}$, calendars survive; away from
that corner, the calendar is exactly what fails.

\subsection{Every Errand Pays Its Way}

\begin{table}[t]
  \centering
  \input{sections/tables/tab2_ablation}
\end{table}
Table~\ref{tab:ablation} removes one component at a time at the base cap.
Spend stays equal within noise, so every loss is an allocation loss, and
every removal costs success. The deepest cut is not a mechanism but the
beliefs: an uncalibrated belief drops \AbCalibDrop{}pp, more than the
deadband gate (\AbDeadbandDrop{}pp) or any single index term, because a
scheduler that orders its doubt badly buys the wrong errands. The lifecycle
carries weight of its own: forcing flips to be permanent ($\mu_a{:=}0$)
gives up \AbRegimeDrop{}pp, and gating the comparator while an item sits in
abeyance gives up \AbComparatorDrop{}pp; closing use is affordable only
because the eyes stay open. Overhead is
one line: scheduling is closed-form, so \sys{} averages
\OvCallsOurs{} model calls per step against \OvCallsZero{} for eager
revalidation, and the scheduler itself makes none
(Table~\ref{tab:overhead}).

\begin{table}[t]
  \centering
  \input{sections/tables/tab3_overhead}
\end{table}
\paragraph{Both regimes meet in one entry.}
\looseness=-1
Figure~\ref{fig:case} traces one item through the paired runs at the base
cap. The item's fuse spec flips at step 641. Eager revalidation had
refreshed it at step 503, then spread the window's budget across the
store; the task arrives at step 720, the stale value serves, and the
delivery is rejected. \sys{} held the same drift under the calibrated
gate, with a deadline to resolve; at step 690 a co-located errand
returned the new spec as a free en-route receipt, and the paired task at
step 722 succeeded with the budget never spent.
Restraint was the winning allocation; the aggregate gaps above are this
trajectory, repeated (recorded timelines: Appendix~\ref{app:cases}).

\begin{figure}[H]
  \centering
  \includegraphics[width=0.9\linewidth]{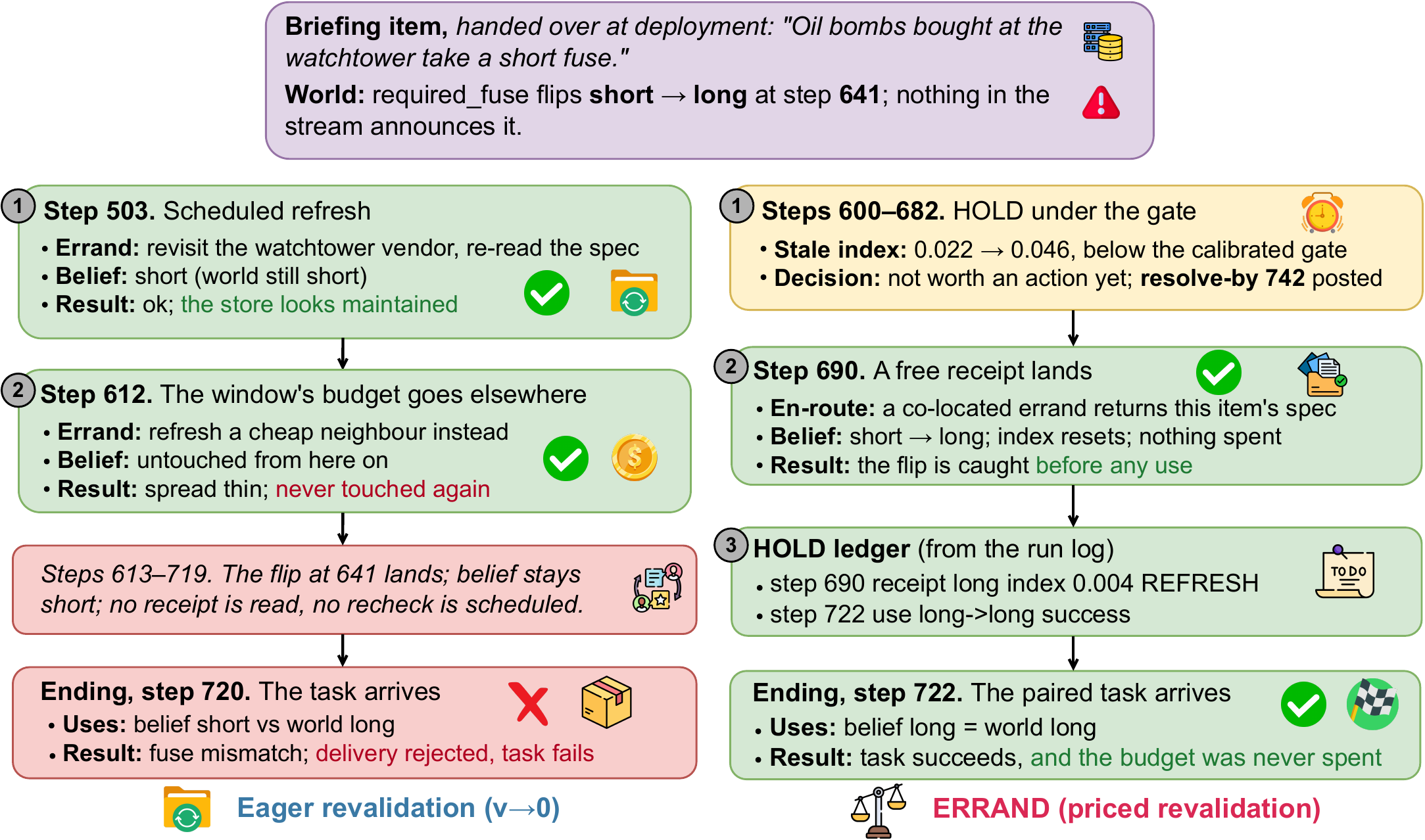}
  \caption{\textbf{One item, two endings.} The same briefing item flips
  mid-stream. Eager revalidation, budget spread thin, serves the stale
  value and fails; \sys{} holds inside the deadband until a co-located
  errand returns the new spec for free, and succeeds.}
  \label{fig:case}
\end{figure}

%% file: sections/tables/tab1_main.tex
\providecommand{\dS}[1]{\textcolor{tabpos}{\bfseries$\uparrow$#1}}
\definecolor{tabpos}{HTML}{1E7A34}
\begin{table}[!b]
\centering\scriptsize
\setlength{\tabcolsep}{2.1pt}
\begin{tabular}{@{}lcccccccccccccc@{}}
\toprule
 & \multicolumn{6}{c}{\cellcolor{SearchBg}\textbf{Budget axis (mid tier, 8B)}} & \multicolumn{2}{c}{\cellcolor{AgenticBg}\textbf{High gap}} & \multicolumn{2}{c}{\cellcolor{MathBg}\textbf{4B floor}} & \multicolumn{2}{c}{\textbf{Average}} & \multicolumn{2}{c}{\cellcolor{DeltaBg}$\Delta$ \textbf{(ours$-$row)}} \\
 & \multicolumn{2}{c}{$b{=}6$} & \multicolumn{2}{c}{$b{=}12$} & \multicolumn{2}{c}{$b{=}24$} & \multicolumn{2}{c}{$b{=}12$} & \multicolumn{2}{c}{$b{=}12$} & & & & \\
\cmidrule(lr){2-3}\cmidrule(lr){4-5}\cmidrule(lr){6-7}\cmidrule(lr){8-9}\cmidrule(lr){10-11}\cmidrule(lr){12-13}\cmidrule(lr){14-15}
Method & ITT\,$\uparrow$ & Cond.\,$\uparrow$ & ITT\,$\uparrow$ & Cond.\,$\uparrow$ & ITT\,$\uparrow$ & Cond.\,$\uparrow$ & ITT\,$\uparrow$ & Cond.\,$\uparrow$ & ITT\,$\uparrow$ & Cond.\,$\uparrow$ & ITT & Cond. & ITT & Cond. \\
\midrule
No maintenance ($\nu{=}\infty$) & 51.1 & 48.9 & 51.1 & 48.9 & 51.1 & 48.9 & 45.4 & 45.5 & 43.8 & 43.6 & 48.5 & 47.1 & \mygreen{54}\dS{16.8} & \mygreen{27}\dS{8.5} \\
Bigger memory, no maint. & 51.2 & 51.9 & 51.2 & 51.9 & 51.2 & 51.9 & 46.5 & 45.8 & 43.6 & 45.0 & 48.7 & 49.3 & \mygreen{53}\dS{16.6} & \mygreen{20}\dS{6.4} \\
\midrule
Fixed-cadence recheck & \underline{58.1} & \underline{53.5} & 61.4 & \underline{56.3} & 64.9 & 57.4 & 56.5 & \underline{55.6} & 46.7 & 43.9 & 57.5 & 53.3 & \mygreen{25}\dS{7.9} & \mygreen{8}\dS{2.4} \\
Equal-budget random & 55.1 & 52.5 & 60.8 & 55.4 & 66.2 & 56.6 & 54.4 & 52.6 & 46.4 & 44.7 & 56.6 & 52.4 & \mygreen{28}\dS{8.8} & \mygreen{11}\dS{3.3} \\
Eager revalidation ($\nu{\to}0$) & 57.8 & 53.4 & \underline{61.9} & 56.0 & \underline{66.2} & \underline{58.0} & \underline{57.2} & 54.6 & \underline{48.5} & \underline{46.0} & \underline{58.3} & \underline{53.6} & \mygreen{23}\dS{7.1} & \mygreen{7}\dS{2.0} \\
\midrule
\textit{Oracle detect-and-filter}$^{\dagger}$ & \textit{53.8} & \textit{70.9} & \textit{53.8} & \textit{70.9} & \textit{53.8} & \textit{70.9} & \textit{47.3} & \textit{63.1} & \textit{45.1} & \textit{49.4} & \textit{50.8} & \textit{65.0} & -- & -- \\
\midrule
\textbf{\sys ($\hat{\nu}_0$)} & \textbf{65.8} & \textbf{56.0} & \textbf{71.9} & \textbf{57.8} & \textbf{71.5} & \textbf{57.9} & \textbf{67.1} & \textbf{59.5} & \textbf{50.6} & \textbf{47.1} & \textbf{65.4} & \textbf{55.7} & -- & -- \\
\bottomrule
\end{tabular}
\caption{\textbf{Under equal action budgets, priced revalidation clears every non-oracle policy on ITT in every setting and on the conditional caliber at every binding budget.} ITT: success over all errand-relevant steps; Cond.: success on judged-domain uses (both \%, $\uparrow$). Frontier arms: five seeds (eleven on the high-gap pair); baselines: three; dispersion and paired intervals: Appendix~\ref{app:results}. Zero-spend rows are cap-invariant, so one run repeats across budget columns. Underline: best non-oracle baseline per column; \colorbox{DeltaBg}{$\Delta$}: average margin of \sys over the row, shaded by size. $^{\dagger}$Oracle reads staleness labels, wins Cond.\ but loses ITT; excluded from $\Delta$.}
\label{tab:main}
\end{table}

%% file: sections/tables/tab2_ablation.tex
\providecommand{\drop}[1]{\textcolor{tabneg}{\,$_{\downarrow#1}$}}
\definecolor{tabneg}{HTML}{B00020}
\centering\small
\setlength{\aboverulesep}{1.2pt}\setlength{\belowrulesep}{1.2pt}
\setlength{\tabcolsep}{8.0pt}
\begin{tabular}{@{}lcccc@{}}
\toprule
Variant & ITT\,$\uparrow$ & Cond.\,$\uparrow$ & Spend & $\Delta$ITT \\
\midrule
\textbf{Full \sys} & \textbf{71.3} & \textbf{61.3} & 10.3\% & -- \\
\midrule
w/o deadband gate & 67.8 & 58.3 & 10.8\% & \myredcell{12}\drop{3.5} \\
coverage sum $\to$ max & 68.1 & 58.5 & 10.7\% & \myredcell{11}\drop{3.2} \\
index: belief only & 66.0 & 57.1 & 11.0\% & \myredcell{19}\drop{5.3} \\
index: age only & 64.7 & 57.0 & 11.0\% & \myredcell{23}\drop{6.6} \\
$\mu_a{:=}0$ (no regime return) & 66.3 & 57.9 & 11.0\% & \myredcell{18}\drop{5.0} \\
uncalibrated belief & 58.3 & 52.6 & 10.8\% & \myredcell{45}\drop{13.0} \\
comparator gated in abeyance & 68.2 & 58.2 & 10.7\% & \myredcell{11}\drop{3.1} \\
\bottomrule
\end{tabular}
\caption{\textbf{Every component pays its way; belief calibration pays most.} Single-component ablations at $b{=}12$: three seeds, every variant paired with the full arm of its own wave. $\Delta$ITT: drop versus full \sys, shaded by size.}
\label{tab:ablation}

%% file: sections/tables/tab3_overhead.tex
\providecommand{\dropc}[1]{\textcolor{tabneg}{\,$_{\uparrow#1}$}}
\definecolor{tabneg}{HTML}{B00020}
\centering\small
\setlength{\aboverulesep}{1.2pt}\setlength{\belowrulesep}{1.2pt}
\setlength{\tabcolsep}{6.0pt}
\begin{tabular}{@{}lcccccc@{}}
\toprule
Policy & \multicolumn{2}{c}{Tokens/step} & LLM calls & Sched.\ & Tokens & $\Delta$calls \\
\cmidrule(lr){2-3}
 & prompt & gen. & /step & ms/step & /errand & \\
\midrule
\textbf{\sys ($\hat{\nu}_0$)} & \textbf{431} & \textbf{64} & \textbf{1.10} & \textbf{40} & \textbf{361} & -- \\
\midrule
Eager ($\nu{\to}0$) & 455 & 72 & 1.40 & 33 & 334 & \myredcell{30}\dropc{0.30} \\
Fixed-cadence & 417 & 60 & 1.53 & 30 & 355 & \myredcell{44}\dropc{0.44} \\
Equal-random & 426 & 62 & 1.44 & 29 & 361 & \myredcell{35}\dropc{0.35} \\
\bottomrule
\end{tabular}
\caption{\textbf{Scheduling never spends a model call.} Per-step accounting at $b{=}12$. The scheduler is closed-form arithmetic ($40$\,ms/step, no LLM invocation); $\Delta$calls: extra LLM calls per step versus \sys, shaded by size, whose targeted errands replace repeated failure recovery.}
\label{tab:overhead}

%% file: sections/5_conclusion.tex
\section{Conclusion}
\looseness=-1
An agent should pay to recheck a belief when resolving that doubt is
worth more than the task step it displaces, and not before. We presented
\sys{}: drift-aware beliefs order the doubt, the deadband gate funds only
value that clears the wage, en-route receipts maintain the task path for
free, and repair is versioned. Priced errands buy the frontier, the
prices are calibrated rather than lucky, and they track the world, not
the calendar. One boundary is the substrate: knowledge arrives as a
handed-over briefing; upkeep of self-consolidated experience is untested.
Future
work prices acquisition, consolidation, and maintenance in one auction, so
the wage that funds an errand decides what deserves to be knowledge at
all.

%% file: sections/6_appendix.tex
\clearpage

\section*{Table of Contents}
\setcounter{tocdepth}{2}
\renewcommand{\contentsname}{Appendix Contents}
\startcontents[appendix]
\printcontents[appendix]{}{1}{}

\counterwithin{table}{section}
\counterwithin{figure}{section}
\counterwithin{equation}{section}
\renewcommand{\theHtable}{app.\thesection.\arabic{table}}
\renewcommand{\theHfigure}{app.\thesection.\arabic{figure}}
\renewcommand{\theHequation}{app.\thesection.\arabic{equation}}

\counterwithin{proposition}{section}
\theoremstyle{plain}
\newtheorem{assumption}{Assumption}[section]
\theoremstyle{definition}
\newtheorem{remark}{Remark}[section]

\clearpage
\input{sections/appendix/app_a_algorithms}
\clearpage
\input{sections/appendix/app_b_closedforms}
\clearpage
\input{sections/appendix/app_c_prereg}
\clearpage
\input{sections/appendix/app_d_details}
\clearpage
\input{sections/appendix/app_e_results}
\clearpage
\input{sections/appendix/app_f_instruments}
\clearpage
\input{sections/appendix/app_g_cases}
\clearpage
\input{sections/appendix/app_h_repro}

\stopcontents[appendix]

%% file: sections/appendix/app_a_algorithms.tex
\section{The \sys{} Protocol: Algorithms and Configuration}
\label{app:algorithms}

The two routines below are the complete executable form of the scheduler.
Algorithm~\ref{alg:step} is one deployment step, run at every step of the
horizon; Algorithm~\ref{alg:life} is the lifecycle rule set it invokes when
evidence arrives. Notation follows Section~\ref{sec:method} and
Appendix~\ref{app:forms}; Table~\ref{tab:config} lists every protocol
constant with its preregistered default. Neither routine touches the frozen
policy, the retriever, or the store's contents at rest: the scheduler owns
only the errand actions and the belief state.

\begin{algorithm}[h]
\caption{\sys{} One Deployment Step
(\colorbox{bandblue}{Harvest}, \colorbox{bandgreen}{Believe},
\colorbox{bandorange}{Price and Gate}, \colorbox{bandpurple}{Settle}, and
\colorbox{bandred}{Meter})}
\label{alg:step}
\textbf{Input:} briefing $M$ with beliefs $\{q_a\}$, wage $\wage$, cap $b$,
retrieval width $k$
\begin{algorithmic}[1]
\STATE \textit{Phase I: Harvest} \Comment{receipts are free and never gated}
\colorbox{bandblue}{
\parbox{0.9\linewidth}{
\STATE execute the step's task action and observe the outcome
\STATE on each atom $a$ the step traversed, the comparator matches the
observation against the recorded value, in every lifecycle state
\STATE a confirmation re-anchors $q_a$ low and counts a receipt into
$N_i$; a refutation is a resolved flip and invokes
Algorithm~\ref{alg:life}
}}

\STATE \textit{Phase II: Believe} \Comment{closed form; no observation spent}
\colorbox{bandgreen}{
\parbox{0.9\linewidth}{
\STATE relax every unobserved belief by \eqref{eq:relax},
$q_a \gets q_a^{\infty} + \left(q_a - q_a^{\infty}\right) e^{-\kappa_a \Delta}$
\STATE compose $q_i = 1 - \prod_{a \in A_i}\left(1 - q_a\right)$; refresh
the durability horizon $\tau_i$ by \eqref{eq:tau}
}}

\STATE \textit{Phase III: Price and gate} \Comment{one comparison, two readings}
\colorbox{bandorange}{
\parbox{0.9\linewidth}{
\FORALL{retrieved candidates $i$ at sites $p$}
\STATE $V \gets \mathrm{VOI}_i\!\left(q_i\right) \hat{u}_i
\min\{\tau_i,\, T - t\}\, \xi\!\left(p,i\right)$; set $c$ to the cheapest
sufficient grade
\ENDFOR
\STATE if $\max_i V/c < \wage$ or the cap binds, log the clause in
\texttt{no\_spend\_reason}; go to Phase~V
\STATE fund the top candidate; extend to co-located items greedily while
each marginal member clears $\wage$
}}

\STATE \textit{Phase IV: Settle} \Comment{the errand resolves an atom}
\colorbox{bandpurple}{
\parbox{0.9\linewidth}{
\STATE run the errand; its receipt confirms or supersedes via
Algorithm~\ref{alg:life}
\STATE stamp the resolve-by date and the revisit gap $\Delta_{\min}$ on
deferred items
}}

\STATE \textit{Phase V: Meter} \Comment{the budget speaks}
\colorbox{bandred}{
\parbox{0.9\linewidth}{
\STATE update the wage by its dual rule at rate $\eta$: up while the cap
binds, decay toward the intrinsic floor while slack
}}
\RETURN updated beliefs, wage, and lifecycle states
\end{algorithmic}
\end{algorithm}

\begin{algorithm}[h]
\caption{Versioned Lifecycle, invoked on evidence about item $i$
(\colorbox{bandpurple}{Repair} and \colorbox{bandgreen}{Redemption})}
\label{alg:life}
\textbf{Input:} a confirming or refuting receipt, free or purchased
\begin{algorithmic}[1]
\STATE \textit{Phase I: Repair} \Comment{a state change, not a deletion}
\colorbox{bandpurple}{
\parbox{0.9\linewidth}{
\STATE on a refutation in service, supersede the recorded value in place,
version the prior value, and restart $q_i$ near certainty
\STATE when $q_i$ crosses the threshold $q_i^{\ast}$, move $i$ to
abeyance: retrieval withholds it and its uses stop, while the comparator
keeps reading every traversal
}}

\STATE \textit{Phase II: Redemption} \Comment{return costs a traversal and a lookup}
\colorbox{bandgreen}{
\parbox{0.9\linewidth}{
\STATE a free receipt that confirms an abeyant item returns it to service
at zero budget
\STATE a funded errand redeems it through the $g$-branch of
$\mathrm{VOI}_i$, the same auction as detection
\STATE a superseded value that recurs costs a version lookup, not
relearning
}}
\RETURN the item's lifecycle state
\end{algorithmic}
\end{algorithm}

Three reading notes. First, everything both routines compute is closed-form
arithmetic over the $k$ retrieved candidates, $O(k)$ per step; scheduling
never spends a model call, and the full per-step price is itemized in
Table~\ref{tab:aovh}. Second, the comparator loop of \textsc{Harvest} runs
before any pricing and is gated by nothing, including abeyance: this is the
algorithmic location of the contract clause that suspension closes use,
never eyes. Third, the gate never declines silently; every no-spend step
logs which clause held, so the deadband of the main text is a logged object,
tabulated in Table~\ref{tab:ansr} and replayed line by line in
Appendix~\ref{app:cases}.

\begin{table}[h]
\centering
\footnotesize
\setlength{\tabcolsep}{4.0pt}
\caption{Protocol constants with preregistered defaults.}
\label{tab:config}
\begin{tabular}{llll}
\toprule
\textbf{Symbol} & \textbf{Role} & \textbf{Default} & \textbf{Source} \\
\midrule
$T$ & horizon (steps) & $2{,}000$; $500$-step prefix re-read & instrument manifest \\
$W$ & warm-up excluded from every metric & $300$ steps & instrument manifest \\
$b$ & errand-action cap per hundred steps & $\{6, 12, 24\}$; dense $\{2, 3\}$ & preregistered axis \\
$c$ & errand cost (actions) & $d + 1$ for site distance $d$ & audited per atom \\
$(\lambda_a, \mu_a)$ & flip-out and flip-back intensity, group A & $(0.006,\ 0.020)$ & drift schedule \\
 & groups B, C & $(0.002,\ 0.0025)$ & drift schedule \\
 & group D & $(0.004,\ 0.006)$ & drift schedule \\
$k$ & retrieval width & release-pinned & code supplement \\
$\eta$ & wage dual-update rate & release-pinned & code supplement \\
$q_i^{\ast}$ & disposition threshold & calibrated at $R1'$, release-pinned & code supplement \\
& seeds & $\{42, \ldots, 46\}$; high tier extends to $11$ & Appendix~\ref{app:power} \\
& backbone & frozen Qwen3-8B; 4B sibling as floor & Appendix~\ref{app:details} \\
\bottomrule
\end{tabular}
\end{table}

%% file: sections/appendix/app_b_closedforms.tex
\section{Closed Forms and Formal Statements}
\label{app:forms}

This appendix derives the four closed forms quoted in the main text, states
the two structural facts the scheduler leans on, and closes with a
literature audit of the modeling choice that separates this world model
from its neighbors. The analysis claims exactness only in the single-atom
regime that instrumentation deliberately makes typical; no theorem-level
novelty is claimed for the scheduling mathematics, which instantiates the
restless-bandit lineage \citep{Whittle1988JAP, NinoMora2026ARXIV}.

\subsection{Preliminaries and Notation}
\label{app:notation}

Table~\ref{tab:notat} collects the symbols; all are inherited from the main
text. Throughout, one item $i$ is fixed, grounded in atoms $a \in A_i$;
the single-atom statements take $A_i = \{a\}$ and drop the subscript where
unambiguous.

\begin{table}[h]
\centering
\small
\caption{Notation used in this appendix.}
\label{tab:notat}
\begin{tabularx}{\textwidth}{lX lX}
\toprule
\textbf{Symbol} & \textbf{Meaning} & \textbf{Symbol} & \textbf{Meaning} \\
\midrule
$\lambda_a, \mu_a$ & flip-out and flip-back intensity & $g_i$ & forgone gain of a suspended item \\
$\kappa_a$ & total flip intensity $\lambda_a + \mu_a$ & $\ell_i$ & misuse loss of a served stale item \\
$q_a$ & staleness belief & $N_i$ & usage receipts of item $i$ \\
$q_a^{\infty}$ & stationary suspicion $\lambda_a / \kappa_a$ & $\hat{u}_i$ & usage rate \\
$q_i^{\ast}$ & disposition threshold & $\tau_i$ & durability horizon \\
$\wage$ & wage (shadow price of the budget) & $\xi(p,i)$ & locality discount, in $(0,1]$ \\
$c$ & errand cost in actions & $b$ & errand cap per hundred steps \\
$T, t$ & horizon and current step & $\mathcal{D}_i$ & deadband of item $i$ \\
\bottomrule
\end{tabularx}
\end{table}

\begin{assumption}[Two-state world]
\label{asm:world}
Each atom flips between its recorded value and its negation as a two-state
Markov chain with intensities $\lambda_a$ (out) and $\mu_a > 0$ (back),
independent across atoms as declared; the correlated-flip suite of
Appendix~\ref{app:results} stresses the declaration.
\end{assumption}

\begin{assumption}[Evidence]
\label{asm:evidence}
Beliefs move only on receipts: en-route observations and funded errands.
Between receipts no information arrives about the atom.
\end{assumption}

\subsection{Belief Relaxation}
\label{app:relax}

\begin{proposition}[Closed-form relaxation]
\label{prop:relax}
Under Assumptions~\ref{asm:world} and \ref{asm:evidence}, the staleness
belief obeys Eq.~\eqref{eq:relax}: $q_a(t + \Delta) = q_a^{\infty} +
(q_a(t) - q_a^{\infty})\, e^{-\kappa_a \Delta}$.
\end{proposition}

\begin{proof}
Between receipts the belief is the occupation probability of the flipped
state, which solves the forward equation $\dot{q}_a = \lambda_a (1 - q_a)
- \mu_a q_a = \lambda_a - \kappa_a q_a$. This linear equation has the
unique fixed point $q_a^{\infty} = \lambda_a / \kappa_a$ and relaxes to it
exponentially at rate $\kappa_a$, which is the stated form.
\end{proof}

\begin{remark}
The corpus default $\mu_a = 0$ collapses $q_a^{\infty}$ to one: every atom
is eventually presumed dead and no suspension is ever worth revisiting.
Every quantity downstream, the stationary suspicion, the durability
horizon, and the redemption branch of the index, inherits its finiteness
from $\mu_a > 0$. Appendix~\ref{app:audit} audits this choice against the
literature.
\end{remark}

\subsection{The Single-Peaked Index and the Deadband Width}
\label{app:voi}

\begin{proposition}[Single peak]
\label{prop:voi}
$\mathrm{VOI}_i(q) = \min\{q \ell_i,\ (1 - q) g_i\}$ vanishes at $q = 0$
and $q = 1$, is concave and piecewise linear, and peaks at
$q^{\dagger} = g_i / (g_i + \ell_i)$ with value
$g_i \ell_i / (g_i + \ell_i)$.
\end{proposition}

\begin{proof}
The first branch increases linearly from zero, the second decreases
linearly to zero, and the minimum of two affine functions is concave. The
branches cross where $q \ell_i = (1 - q) g_i$, giving $q^{\dagger}$ and,
by substitution, the peak value.
\end{proof}

\begin{proposition}[Deadband width, Eq.~\eqref{eq:band}]
\label{prop:band}
Let a check on a single-atom item resolve its doubt at cost $c$, and let
the gate fund it when $\mathrm{VOI}_i(q)\, N_i \ge \wage c$. The unfunded
set of beliefs is a union of two intervals at the ends of $[0,1]$ with
total width
$\lvert \mathcal{D}_i \rvert = \wage c\, (g_i + \ell_i) / (g_i \ell_i N_i)$,
provided the peak clears the wage.
\end{proposition}

\begin{proof}
On the rising branch the gate declines while $q \ell_i N_i < \wage c$,
an interval of width $\wage c / (\ell_i N_i)$ at the fresh end. On the
falling branch it declines while $(1 - q) g_i N_i < \wage c$, an interval
of width $\wage c / (g_i N_i)$ at the stale end. When the peak clears the
wage these are the only declines, and their widths sum to the stated
expression.
\end{proof}

\begin{remark}
The deadband is not a band around the peak: it is the two certainty ends.
An item almost surely fresh and an item almost surely stale are both
priced at zero, which is the sentence ``certainty is free at both ends''
in closed form. Costlier checks ($c$) widen both ends; heavier use
($N_i$) narrows them.
\end{remark}

\subsection{The Durability Horizon}
\label{app:tau}

\begin{proposition}[Re-crossing time, Eq.~\eqref{eq:tau}]
\label{prop:tau}
After a resolution restarts the belief at $q_i$, the time until it crosses
the disposition threshold $q_i^{\ast}$ is
$\tau_i = \kappa_i^{-1} \ln \bigl( \lvert q_i^{\infty} - q_i \rvert /
\lvert q_i^{\infty} - q_i^{\ast} \rvert \bigr)$
when $q_i^{\ast}$ lies between $q_i$ and $q_i^{\infty}$, and
$\tau_i = \infty$ otherwise.
\end{proposition}

\begin{proof}
Solve Proposition~\ref{prop:relax} for the crossing time: the gap to the
stationary point contracts by $e^{-\kappa_i \Delta}$, so the threshold is
reached when the contraction factor equals the ratio of the two gaps,
giving the logarithm. If the threshold does not lie between the running
belief and $q_i^{\infty}$, the relaxation path never meets it.
\end{proof}

\begin{remark}[The $\tau$ ordering is a prediction]
Items whose stationary suspicion sits close to the threshold have
$\tau_i \to \infty$: their answers never expire by drift alone, and the
index, which multiplies by $\min\{\tau_i, T - t\}$, buys them first. The
fluid surrogate $1/\kappa_i$ cannot express this, which is why the
surrogate comparison in Appendix~\ref{app:results} is a preregistered
check rather than a convenience.
\end{remark}

\subsection{The Value of the Free Channel, Numerically}
\label{app:numerics}

The independence of the two evidence channels is not assumed; it is
priced. A two-state value iteration (exact on the closed two-state chain;
no grid) evaluates the stationary value of three evidence policies under
the same wage: no receipts at all, use receipts only, and use plus
en-route receipts. Table~\ref{tab:bnum} reports the fixed points: the
ladder is monotone at every $(\kappa, c)$ cell, pooling to
$\ApNumLadder$, and the en-route increment is invariant to the errand cost
$c$ while the purchased channel scales with it, because passage is free by
construction. This is the modeling value of never gating the comparator,
measured before any deployment run; the deployment-side twin is the gated
comparator ablation of Table~\ref{tab:aabl}.

\begin{table}[h]
\centering
\small
\caption{Stationary value of three evidence policies (two-state value
iteration).}
\label{tab:bnum}
\input{sections/appendix/tables/tab_b_numerics}
\end{table}

\subsection{Zero Shadow Price Outside the Binding Regime}
\label{app:shadow}

\begin{proposition}[Scope of the ordering premium]
\label{prop:shadow}
Consider the budget as the constraint of a linear relaxation of the
errand-selection problem, with $\wage$ its multiplier. If the cap does not
bind at the optimum, complementary slackness forces the multiplier to the
intrinsic floor, and any policy that funds every candidate clearing the
floor attains the optimum: the marginal value of ordering vanishes.
Ordering carries a strictly positive premium only where the budget binds
and some priced doubt goes unfunded.
\end{proposition}

\begin{proof}[Proof sketch]
With a slack cap the budget constraint is inactive, so its multiplier is
zero above the floor and the relaxation decouples across candidates: each
is funded if and only if its value per action clears the floor,
independently of the others. Any ordering of funded candidates yields the
same objective, so the premium of ordering over any other spend rule with
the same funded set is zero. When the cap binds, the funded set is a
strict subset of the floor-clearing set, and which subset is chosen is
exactly what ordering decides.
\end{proof}

\begin{remark}[Scope, stated plainly]
This is a statement about the relaxation, not a regret bound; we use it as
a preregistered scope condition, not a guarantee. Its two empirical
faces appear in Appendix~\ref{app:results}: on the binding layer of the
coverage grid the ordering premium grows monotonically with the coverage
gap, and on the non-binding layer it is indistinguishable from zero,
$\ApLifeboatNB$\,pp pooled. The deployment rule of the main text is this
proposition read aloud: pricing matters exactly where receipts are scarce
and the budget is tight.
\end{remark}

\subsection{The Flip-Back Audit}
\label{app:audit}

Table~\ref{tab:blam} audits how the memory-maintenance corpus and the two
adjacent classical lineages treat regime return, the $\mu_a$ axis of
Assumption~\ref{asm:world}. Within LLM agent memory, staleness is modeled
as a monotone process: one-way invalidation, frozen stores, or deletion on
contradiction. Versioning does exist in places, and three systems draw a
revival edge, a ``revalidatable'' or reinstatable state, but in each case
the edge is drawn without a trigger, a price, or a demonstration. The
classical refresh lineage prices recheck scheduling but in the
$\mu_a \to 0$ limit of Eq.~\eqref{eq:relax}; the restless-bandit line
models spontaneous recovery outright, for interventions rather than for
memory. What no audited system does is the combination this paper runs on:
model the return, version the repair, and make detection and redemption
compete for the same wage in one currency.

\begin{table}[h]
\centering
\footnotesize
\caption{How prior models treat flip-back. $^{\ast}$A revival edge is
stated but carries no trigger or price.}
\label{tab:blam}
\input{sections/appendix/tables/tab_b_lambda}
\end{table}

%% file: sections/appendix/tables/tab_b_numerics.tex
\setlength{\tabcolsep}{5.0pt}
\renewcommand{\arraystretch}{1.12}
\rowcolors{3}{white}{zebra}
\begin{tabular}{@{}lrrrrrr@{}}
\toprule
\textbf{Evidence channels} & \multicolumn{2}{c}{\cellcolor{grpteal}$\kappa{=}0.18$} & \multicolumn{2}{c}{\cellcolor{grpteal}$\kappa{=}0.4$} & \multicolumn{2}{c}{\cellcolor{grpteal}$\kappa{=}1.04$} \\
\cmidrule(lr){2-3}\cmidrule(lr){4-5}\cmidrule(lr){6-7}
 & $c{=}1$ & $c{=}5$ & $c{=}1$ & $c{=}5$ & $c{=}1$ & $c{=}5$ \\
\midrule
no receipts at all & 0.060 & 0.060 & 0.060 & 0.060 & 0.060 & 0.060 \\
use receipts only & 0.258 & 0.258 & 0.163 & 0.163 & 0.103 & 0.103 \\
\rowcolor{herobg} \textbf{use $+$ en-route receipts} & 0.407 & 0.358 & 0.259 & 0.194 & 0.140 & 0.061 \\
\bottomrule
\end{tabular}

%% file: sections/appendix/tables/tab_b_lambda.tex
\setlength{\tabcolsep}{3.5pt}
\renewcommand{\arraystretch}{1.15}
\rowcolors{2}{white}{zebra}
\begin{tabular}{@{}p{3.7cm}p{4.7cm}p{2.1cm}p{2.7cm}@{}}
\toprule
\textbf{System} & \textbf{Staleness model} & \textbf{Return ($\mu_a$)?} & \textbf{Repair} \\
\midrule
GLOVE \citep{Yin2026ARXIV} & reactive mismatch on re-encounter; no temporal model & no & delete and replace \\
STALE \citep{Chao2026ARXIV} & one-way LLM-adjudicated invalidation & no & archive; read-time block \\
AdaMEM \citep{Zhang2026ICML} & store frozen after construction & n.a. & regenerate a transient strategy \\
SkillGuard \citep{Fan2026ARXIV} & one-shot contract check at ingestion & no & one-shot rewrite \\
Feedback loops \citep{Cui2026ICMLW} & outcome bookkeeping; hand-set threshold & edge only$^{\ast}$ & deprecate; no reinstate \\
MemTX \citep{Li2026ARXIV} & declared validity interval, never re-evaluated & edge only$^{\ast}$ & typed cascade; supersede \\
Kumiho \citep{Park2026ARXIV} & graph-internal deprecation; no temporal model & edge only$^{\ast}$ & versioned supersede \\
Stale-memory audit \citep{Sun2026ARXIV} & injected flips, truth in prompt & no & delete-only filter \\
Mem0 \citep{Chhikara2025ARXIV} & write-time contradiction only & no & delete on contradiction \\
Letta \citep{Packer2023ARXIV} & consolidation and structure hygiene & n.a. & none \\
Zep \citep{Rasmussen2025ARXIV} & ingestion-triggered edge invalidation by timestamp & no & supersession by new edges \\
SafeCommit \citep{Akewar2026ARXIV} & per-decision conformal score, recomputed each step & n.a. & none: commit, probe, or fall back \\
MirrorCraft \citep{Gao2026ARXIV} & prior-mismatch benchmark; rules constant within an episode & n.a. & none: memory cleared across episodes \\
Supersede \citep{Patel2026ARXIV} & bounded notes field; a superseded value persists unless overwritten & no & overwrite only \\
StateAuditor \citep{Sun2026ARXIVb} & decay-based suspicion forbidden; monotone chronology gate & no & regenerate the response; memory read-only \\
Web refresh \citep{Cho2000SIGMOD,Cho2003TODS,BusaFekete2025ARXIV} & Poisson change; freshness $e^{-\lambda\tau}$ & no (the $\mu_a{\to}0$ limit) & scheduled re-crawl \\
Restless adherence \citep{NinoMora2026ARXIV} & two-state belief chain, spontaneous recovery & \textbf{yes} (not a memory system) & allocates interventions \\
\bottomrule
\end{tabular}

%% file: sections/appendix/app_c_prereg.tex
\section{Preregistered Analysis Plan and Power}
\label{app:power}

Every headline estimand in this paper was frozen as a named quantity with
an interpretation band and a fallback reading before the first full-wave
GPU hour, in the study's request document; the runs then filled the bands
mechanically. This appendix states the estimands, reports every band
against its outcome, and closes with the design-stage power forecast and
what the realized study delivered against it.

\subsection{Estimands and Inference}
\label{app:estimands}

All rates are computed on the post-warm-up window of each run. ITT is
success over all errand-relevant steps and is primary throughout. The
conditional caliber isolates the store's contribution and is emitted in
three forms: the \emph{judged-domain} rate, computed on the preregistered
domain of knowledge-bearing uses and pinned before any arm ran; a
\emph{full-domain audit} over all encounters; and a \emph{literal audit}
that scores exact value matches. Every table draws the judged rate; the
two audits ship with every deliverable, and a domain switch is permitted
only at the line where an estimator is computed. Two guards are mechanical:
the judged and ITT point estimates of any comparison must agree in sign,
since a mechanism that hides items rather than repairing them inflates the
conditional rate while ITT falls; and a conditional cell with fewer than
$30$ stale encounters reports not-available rather than a number.

Inference follows the pairing design. The two arms of every headline
comparison run on the same machine for each seed and cap, servers rotate
across seeds, and the seed is the unit of inference: intervals are $95\%$
paired $t$ over seed-level means, and the item-level matched pairs of the
coverage study are resolved by seed-clustered inference with a bootstrap
cross-check. Budget claims attach only to caps proven binding by logged
cap-hit events; comparisons on slack caps are labeled descriptive.

\subsection{Bands and Outcomes}
\label{app:bands}

Table~\ref{tab:cpre} reports all nine preregistered estimands. Seven land
inside their bands. Two exercised the fallback reading written for them in
advance, and we adopt those readings verbatim. The dominance count
resolves to $1/3$ under its strict three-leg criterion: at $b{=}6$ both
arms pin the cap, so measured spend differs by \ApDomSpendSix{}\,pp of
steps and the spend leg fails on that margin; at $b{=}24$ the conditional
caliber ties at \ApDomCondLoose{}\,pp. Dominance at the base cap holds on
all three legs, and the preregistered fallback scopes the strict claim to
that cap, where the main text anchors its headline margin. The calibration ladder clears
its upper segment decisively, \ApCalibSegTwo{} for pricing on calibrated
beliefs, while the lower segment, \ApCalibSegOne{}, is positive but not
resolved at three paired seeds; the fallback reads the ladder as a
single-sided claim (Figure~\ref{fig:apow}c), and the component view of
Table~\ref{tab:aabl}, where the uncalibrated-belief arm is the deepest cut
in the panel, carries the calibration story with the resolution the ladder
lacks.

\begin{table}[p]
\centering
\scriptsize
\caption{Preregistered interpretation bands and outcomes. Bands were
frozen before the full waves; verdicts are mechanical.}
\label{tab:cpre}
\input{sections/appendix/tables/tab_c_prereg}
\end{table}

\subsection{Power, Forecast and Realized}
\label{app:powerfc}

The matched-pair study was sized before it ran. From calibration rates,
the forecast fixed the pair yield and the half-width of the matched
conditional gap per tier and seed count, with an upgrade rule preregistered
at $480$ pairs; the realized high-tier cohort delivered \ApPairsRealized{}
pairs across eleven seeds, and the realized seed-clustered half-width,
\ApCondHalf{}\,pp, beats the forecast at the same row
(Table~\ref{tab:cpow}). The two cohorts the design nested are both hard:
the eight-seed gate cohort reads \ApCondEight{} and the eleven-seed
extension \ApCondMatched{}, same sign, overlapping intervals, so the
matched conditional gap is robust to the seed schedule. The forecast also
fixed the budget geometry: caps of $6$ and $12$ were predicted binding and
$24$ slack. The realized cap-hit ledger of Table~\ref{tab:afro} is
sharper than the forecast: at $6$ and $12$ every run of both arms
exhausts the cap, while at $24$ every eager run still does and no \sys{}
run ever does, so the loose cap is slack only for the policy that knows
when not to spend.

Figure~\ref{fig:apow} collects the design study beside its realized
readings.

\begin{table}[p]
\centering
\small
\caption{Design-stage pair-yield forecast and the realized cohort.}
\label{tab:cpow}
\input{sections/appendix/tables/tab_c_power}
\end{table}

\begin{figure}[p]
\centering
\begin{subfigure}{0.39\linewidth}
  \centering
  \includegraphics[width=\linewidth]{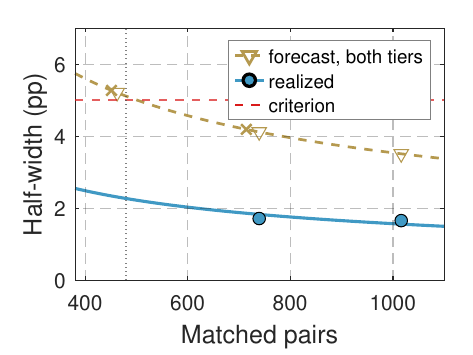}
  \caption{Half-width, forecast and realized}
\end{subfigure}\hspace{0.03\linewidth}%
\begin{subfigure}{0.39\linewidth}
  \centering
  \includegraphics[width=\linewidth]{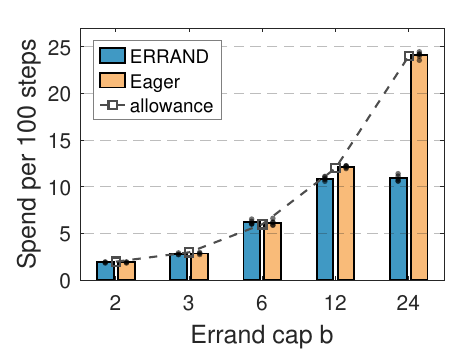}
  \caption{Spend against the allowance}
\end{subfigure}

\vspace{4pt}
\begin{subfigure}{0.39\linewidth}
  \centering
  \includegraphics[width=\linewidth]{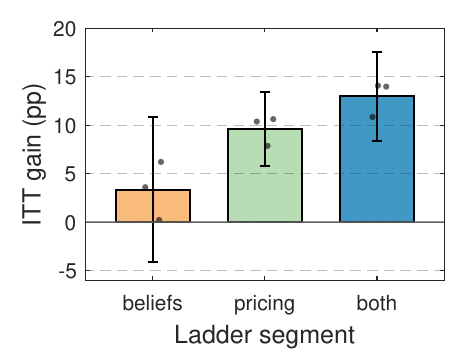}
  \caption{Ladder segments, paired}
\end{subfigure}\hspace{0.03\linewidth}%
\begin{subfigure}{0.39\linewidth}
  \centering
  \includegraphics[width=\linewidth]{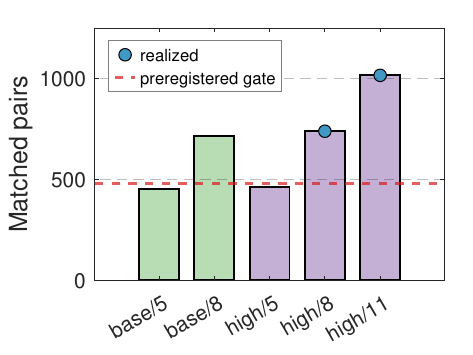}
  \caption{Pair yield against the gate}
\end{subfigure}
\caption{The design study, realized. (a)~The forecast is one $c/\sqrt{n}$
law, so both tiers share a curve; the realized cohorts halve its constant.
Dotted, the $480$-pair gate; dashed, the $5$\,pp criterion. (b)~Measured
spend against the allowance: both arms empty the purse until the cap stops
binding for \sys{} alone. (c)~The two ladder segments, paired by seed;
whiskers, $95\%$ paired $t$. (d)~Forecast pair yield against the gate
(dashed); dots, realized cohorts.}
\label{fig:apow}
\end{figure}

%% file: sections/appendix/tables/tab_c_prereg.tex
\setlength{\tabcolsep}{3.0pt}
\renewcommand{\arraystretch}{1.15}
\rowcolors{2}{white}{zebra}
\begin{tabular}{@{}p{2.0cm}p{2.6cm}p{1.2cm}p{2.5cm}p{3.2cm}p{1.2cm}@{}}
\toprule
\textbf{Estimand} & \textbf{Reads} & \textbf{Carrier} & \textbf{Preregistered band} & \textbf{Outcome} & \textbf{Verdict} \\
\midrule
\texttt{StaleCost} & stale collapse on the frozen arm & R1$'$ & $\ge 40$\,pp & $+79.5$\stdv{0.6}\,pp & in band \\
\texttt{FrontierGap} & ITT, \sys{} vs.\ eager @ $b{=}12$ & W1-3 & CI $> 0$ and $\ge {+}3$\,pp & $+10.0$\,pp [+8.6, +11.4] & in band \\
\texttt{SpendSave} & spend saved at the same cap & W1-3 & $\ge 5\%$, cond.\ within $1.5$\,pp & $10.3$\stdv{1.9}\% ($\Delta$cond $+1.8$\,pp [-0.8, +4.5]) & in band \\
\texttt{DominanceCell} & caps where \sys{} dominates eager & W1-3 & $\ge 2/3$ & $1/3$ & 1/3: single-cap claim \\
\texttt{CondMatched} & matched conditional gap, high tier & W1-5$+$W3-5 & CI $> 0$ at half-width $\le 5$\,pp & $+4.9$\,pp [+3.3, +6.6] & in band \\
\texttt{GapPrice} & ordering premium, high minus mid tier & W1-5$+$W1-3 & CI $> 0$ & $+3.8$\,pp [+1.8, +5.8] & in band \\
\texttt{CalibLadder} & blind $<$ eager $<$ \sys{} at equal spend & W2-1 & both segments CI $> 0$ & $+3.3$\,pp [-4.1, +10.8]; $+9.6$\,pp [+5.8, +13.4] & upper segment only \\
\texttt{PlateauFlat} & cond.\ $b{=}24$ vs.\ $b{=}12$; no extra spend & W1-3 & $\lvert\Delta\rvert \le 1.5$\,pp & $+0.0$\,pp [-2.5, +2.5]; spend $+0.0$\,pp & in band \\
\texttt{RevivePremium} & relearning cost of $\mu_a{:=}0$ & W2-1 & $> 0$, measurable & $+5.0$\,pp [+1.4, +8.6] & in band \\
\bottomrule
\end{tabular}

%% file: sections/appendix/tables/tab_c_power.tex
\setlength{\tabcolsep}{5.5pt}
\renewcommand{\arraystretch}{1.12}
\rowcolors{2}{white}{zebra}
\begin{tabular}{@{}llrrl@{}}
\toprule
\textbf{Tier} & \textbf{Seeds} & \textbf{Pairs (forecast)} & \textbf{Half-width (pp)} & \textbf{Gate at 480} \\
\midrule
mid & 5 & 451 & 5.27 & below 480 upgrade \\
mid & 8 & 714 & 4.19 & meets \\
high & 5 & 462 & 5.21 & below 480 upgrade \\
high & 8 & 739 & 4.12 & meets \\
high & 11 & 1016 & 3.51 & meets W3\mbox{-}5 \\
\midrule
\rowcolor{herobg} realized & 11 & \textbf{1016} & \textbf{1.66} & met \\
\bottomrule
\end{tabular}

%% file: sections/appendix/app_d_details.tex
\section{Environment, Instrumentation Principles, and Experimental Details}
\label{app:details}

This appendix records the preregistered configuration of the study: the
errand world and its drift, the briefing and the instrumentation
principles it enforces, the six pitfalls the evaluation design guards
against, the compared baselines, and the adjudication stack. Every
quantity here is a design constant fixed before the runs, except the
calibration readings of Table~\ref{tab:dgat}, which are the instrument's
own acceptance test; measured outcomes live in the main text and in
Appendix~\ref{app:results}.

\subsection{The Errand World and the Task Stream}
\label{app:world}

The main environment is a scripted tool-use world of sites, paths,
non-player characters, and flags, built on the AgentOdyssey harness with
custom step rules that drive drift (Appendix~\ref{app:repro}). A frozen
agent serves dispatch orders over $T = 2{,}000$ steps: \emph{go} orders,
whose completion requires reaching a site, and \emph{fetch} orders, whose
completion requires obtaining an object and delivering it to a drop point
or a character. The fetch form is load-bearing: completion conditions are
the only lever that puts an action verb into the agent's task graph, so
fetch orders are what make off-path pick-up and purchase atoms reachable
at all (Section~\ref{app:exo}). Order text names only what to get and
where to deliver it. Two preregistered dispatch tiers set the coverage
gap: the base tier issues $82$ orders per run (A$20$/B$16$/C$22$/D$16$/F$8$
across atom groups), the high tier $110$ with the off-path groups doubled
and the on-path groups halved (A$10$/B$8$/C$44$/D$32$/F$16$), holding the
exogenous share of the window near three-fifths in both. Every errand's
cost obeys the audited contract $c = d + 1$ for site distance $d$.

\subsection{Drift and the Roster}
\label{app:drift}

The briefing grounds in $26$ atoms in five groups spanning path
connectivity, object values, and access flags; one atom sits outside the
judged domain by a preregistered connectivity guard, leaving
\ApRosterAtoms{} judged. Atoms flip under the four intensity schedules of
Table~\ref{tab:config}, spanning on- and off-path sites at fast and slow
dwell; nothing in the stream announces a flip. Flip pressure is accounted
at the \emph{effective} rate, nominal intensity times one minus the
interception rate that path-locking imposes, and the world's topology
keeps interception bounded: each site cluster carries at most one
oscillating edge or at least two outer gates. Table~\ref{tab:dros}
summarizes the roster as the deployed agent actually meets it.

\begin{table}[h]
\centering
\small
\caption{Roster by atom group, per run on the \sys{} arm at the base cap.}
\label{tab:dros}
\input{sections/appendix/tables/tab_d_roster}
\end{table}

\subsection{The Briefing and the Substrate Principle}
\label{app:substrate}

The briefing is seeded exogenously: the same $26$ consolidated items are
handed to every arm at deployment, identically worded, rather than
accumulated from the agent's own trajectories. This is an instrumentation
principle, not a convenience. Knowledge a system acquires through its own
trajectories is approximately self-maintaining, because acquisition routes
and revisit routes coincide and en-route receipts refresh what the agent
already frequents. A trajectory-seeded store therefore couples an item's
existence to its coverage, and turning coverage down measures ignorance,
not staleness. Handed-over knowledge, briefings, runbooks, another agent's
consolidated experience, is the substrate on which maintenance economics
is even defined: existence decouples from traffic, staleness accumulates
silently off-path, and the first free receipt arrives bundled with the
first, possibly failing, use. Seeding is not a treatment: all arms receive
the same briefing, no receipt channel is weakened, and the substrate
itself must pass the gates of Table~\ref{tab:dgat} before any mechanism
is read.

\subsection{Three Layers of Exogeneity}
\label{app:exo}

For the world to be a valid instrument for staleness, three conditions
hold jointly, each supplied by the design and none by the mechanism under
test. \emph{Seeding}: items exist independently of the agent's history
(Section~\ref{app:substrate}). \emph{Dispatch}: tasks arrive as orders, so
the agent's presence at knowledge-bearing sites is set by the stream, not
self-proposed, and the coverage gap is a dial rather than an emergent
accident. \emph{Completion verbs}: order completion conditions pull the
required action verbs into the task graph, closing the instrumental gap in
which an agent reaches a site but never performs the verb the knowledge is
about. The third layer has a measured boundary: with fetch orders off,
off-path fetch-verb atoms record \ApClosureOff{} uses per order, and
turning the orders on releases them to \ApClosureOn{}
(Table~\ref{tab:aclo}), which is why instrumental closure enters the
regime map as its own axis. All three layers are arm-invariant: order
texts are whitelisted, carry no truth or freshness fields, and a hard
assertion checks the whitelist at the first step of every run.

\subsection{Six Pitfalls of Evaluating Memory Maintenance}
\label{app:pitfalls}

The evaluation design guards against six failure modes, each of which we
either met in calibration or constructed deliberately; stated as design
rules:

\begin{enumerate}[leftmargin=1.6em, itemsep=1pt]
\item \emph{Denominator migration.} A mechanism that hides bad items
inflates its own conditional rate by shrinking the denominator. The judged
domain is pinned before arms run, both audits ship with every deliverable,
and conditional and ITT point estimates must co-sign
(Section~\ref{app:estimands}).
\item \emph{Errands in the denominator.} Counting maintenance actions as
encounters makes the arms' denominators non-comparable; only task uses
enter the conditional denominator.
\item \emph{Pseudo-replication.} Item-level pairs inflate $n$ while the
independent unit is the seed; all matched inference is seed-clustered.
\item \emph{Ignorance--staleness degeneracy.} Trajectory-seeded stores
make coverage and existence one variable
(Section~\ref{app:substrate}); exogenous seeding is the necessary
instrument condition for the staleness construct.
\item \emph{Partisan accounting.} Instrumentation defects concentrate in
the stratum where the mechanism's value lives, off-path atoms, and a
global smoke test cannot see them because on-path traffic dominates it.
Accounting audits run stratified by the treatment axis before every wave
(Appendix~\ref{app:instruments}).
\item \emph{Endogenous gate metrics.} Using an outcome quantity as a
pre-run gate lets paid probing feed the gate; every pre-run gate reads a
constructive proxy that the dial sets and no policy can move.
\end{enumerate}

\subsection{Compared Baselines}
\label{app:baselines}

Seven deployed policies and five mechanism arms run on one substrate.
Two \emph{zero-spend anchors} bracket the memory question: the frozen
briefing with no maintenance, and a bigger store, the same briefing plus
twelve legacy items, with paired judging restricted to the shared
twenty-six so the comparison stays item-matched. Three \emph{unpriced
spenders} pay the same budget without prices: eager revalidation, which
funds the most suspect item whenever the cap allows; a fixed-cadence
recheck that scans two items per period on a clock; and equal-budget
random errands, which keep the gate's spend rate but choose targets
uniformly, the undirected-probing pattern of reactive-repair designs
\citep{Yin2026ARXIV}. Two \emph{priced ceilings} are labeled as such: an
oracle detect-and-filter that reads ground-truth staleness labels at zero
action cost and masks stale items at retrieval, the read-time--filter
family \citep{Sun2026ARXIV}, and a known-rate reference on the drift sweep
that receives the true $\kappa$. The drift sweep adds the \emph{literature
policy}, refresh cadences set from published half-life priors, the
field's default answer to staleness. Five \emph{mechanism arms} reuse the
full scheduler minus one component each, including a Whittle closed-form
allocator \citep{NinoMora2026ARXIV} and a myopic value-rate rule on the
allocator axis; all twelve are specified operationally in
Appendix~\ref{app:instruments} and measured in Appendix~\ref{app:results}.

\subsection{Adjudication and Serving}
\label{app:adjudication}

Ground truth is mechanical end to end. The harness's step rules emit an
event log of every atom flip, so truth at any step is a lookup; order
completion is a scripted predicate; and no LLM judges anything, matching
the scheduler's own zero-model-call discipline. The backbone is a frozen
Qwen3-8B served identically for every arm, with a 4B sibling for the
capability-floor axis. Two identical servers host all runs: the two arms
of every comparison run sequentially on the same machine for each seed and
cap, and seeds rotate across servers, so machine effects difference out.
Each run starts by writing a fingerprint of its instrument tree, dispatch
table, drift schedule, and seed; runs are comparable only under equal
fingerprints.

\subsection{Instrument Calibration}
\label{app:gates}

Before any mechanism arm ran, the frozen no-maintenance arm had to prove
the substrate on five gates per tier: enough distinct items in use, enough
stale encounters to judge, dwell long enough to matter, dispatch
compliance, and matched-pair yield. Table~\ref{tab:dgat} reports the
readings: all gates pass on both tiers. The same runs price the failure
mode the study is about: on items whose atoms have flipped, conditional
success collapses by $\ApStaleCost$\,pp against their fresh
complement, the entry every maintenance policy in this paper is buying
back.

\begin{table}[h]
\centering
\small
\caption{Substrate gates at calibration, against their thresholds.}
\label{tab:dgat}
\input{sections/appendix/tables/tab_d_gates}
\end{table}

%% file: sections/appendix/tables/tab_d_roster.tex
\setlength{\tabcolsep}{5.0pt}
\renewcommand{\arraystretch}{1.12}
\rowcolors{2}{white}{zebra}
\begin{tabular}{@{}lrrrrrr@{}}
\toprule
\textbf{Group} & \textbf{Atoms} & \textbf{$c$ range} & \textbf{Orders/run} & \textbf{Uses/run} & \textbf{Receipts/run} & \textbf{Stale share} \\
\midrule
A & 8 & 1--2 & 20.0 & 43.5 & 58.7 & 13.8\% \\
B & 8 & 2--3 & 16.0 & 98.5 & 182.7 & 15.3\% \\
C & 3 & 3--3 & 22.0 & 17.1 & 28.2 & 15.9\% \\
D & 2 & 4--5 & 14.0 & 10.5 & 25.8 & 18.2\% \\
F & 4 & 5--7 & 8.0 & 0.2 & 15.7 & 0.0\% \\
\bottomrule
\end{tabular}

%% file: sections/appendix/tables/tab_d_gates.tex
\setlength{\tabcolsep}{4.5pt}
\renewcommand{\arraystretch}{1.12}
\rowcolors{2}{white}{zebra}
\begin{tabular}{@{}llrrrrrrl@{}}
\toprule
\textbf{Tier} & \textbf{Gate} & B1 & B2 & B3 & B4a & B4b$'$ & B5 & \textbf{Verdict} \\
 & threshold & $\ge 12$ & $\ge 30$ & $\ge 150$ & $\ge 75\%$ & $\ge 2$ & $\ge 8$ & \\
\midrule
high & measured & 23.3\stdv{0.6} & 69.7\stdv{13.5} & 199\stdv{4} & 85.4\stdv{6.2} & 4.7\stdv{0.6} & 16.0\stdv{7.2} & PASS (3 seeds) \\
mid & measured & 22.2\stdv{1.5} & 77.0\stdv{7.6} & 202\stdv{13} & 85.6\stdv{3.3} & 5.4\stdv{0.9} & 19.8\stdv{3.8} & PASS (5 seeds) \\
\bottomrule
\end{tabular}

%% file: sections/appendix/app_e_results.tex
\section{Extended Results}
\label{app:results}

This appendix expands each measurement summarized in the main text into
its full per-setting, per-seed form. Unless a table states otherwise, the
anchor setting is the base dispatch tier at $b{=}12$ on the frozen 8B
backbone; every cell is computed by script from the delivered run atoms,
and $\pm$ is one standard deviation across runs. The tables follow the
order in which the paper introduces its instruments: the frontier and its
inaction ledger, the baseline grid, the coverage-gap ladder, ablations
and allocators, the coverage grid with its zero-shadow-price layer, the
drift envelope, correlated flips, the regime axes, the horizon, the
freshness panorama, and the overhead account.

\subsection{The Frontier at Every Cap}
\label{app:res-frontier}

\textbf{Priced errands hold the spend axis end to end, and only the
priced policy leaves the loose cap unspent.}
Table~\ref{tab:afro} is the full frontier grid, including the
deep-binding caps $b \in \{2, 3\}$, both uncapped endpoints, and each
arm's realized judged domain per run. \sys{}
clears eager revalidation on ITT at every cap and on the conditional
caliber at every binding one, and the cap-hit
ledger localizes the budget geometry: at $b \le 12$ every run of both
arms exhausts its cap, while at $b{=}24$ every eager run still does and
no \sys{} run ever does. Uncapped, the contrast is starker: eager revalidation
converts most steps into errands, while \sys{} settles at the same pace
it chooses under the loose cap. The wage, not the cap, is doing the
governing, which is the main text's claim read off a ledger.

\begin{table}[h]
\centering
\scriptsize
\caption{Frontier grid: both spending arms at every cap, with the frozen
arm as floor. Cap hit: runs that ever exhaust the trailing budget. Five
seeds at $b \in \{6, 12, 24\}$ and on the floor; three at the deep caps
and uncapped.}
\label{tab:afro}
\input{sections/appendix/tables/tab_e_frontier}
\end{table}

\subsection{The Inaction Ledger}
\label{app:res-nsr}

\textbf{Every unspent step names its clause, and the clauses separate
restraint from starvation.}
Table~\ref{tab:ansr} decomposes no-spend steps by logged reason. The cap
clause logs any step whose selected errand the trailing allowance cannot
fund at that moment; exhaustion, the event Table~\ref{tab:afro} counts,
is its limiting case. For
\sys{} at the loose cap, \ApNsrHatLooseActive\% of unspent steps are
active clauses, the index at zero or the gate below the wage: the
deadband holding as designed. For eager revalidation at the tight cap,
\ApNsrZeroTightCap\% of unspent steps are the cap itself: a policy that
wanted to spend and could not. The same table is the audit trail behind
the plateau reading of Table~\ref{tab:cpre}, and Case~3 of
Appendix~\ref{app:cases} replays one such window line by line.

\begin{table}[h]
\centering
\small
\caption{No-spend steps by clause, share of each run's no-spend total.}
\label{tab:ansr}
\input{sections/appendix/tables/tab_e_nsr}
\end{table}

\subsection{Baselines in Every Setting}
\label{app:res-baselines}

\textbf{No unpriced spender closes the gap anywhere, and spend does not
predict rank.}
Table~\ref{tab:abas} extends the main table's grid with each policy's
measured spend. The unpriced spenders pay \sys{}-comparable budgets at
the binding caps, eager revalidation alone spends the loose cap in full,
and none of it changes the order: what separates policies is not how much
they spend but what the spend buys. The two zero-spend anchors floor
every column, and the bigger store never buys back the gap, so scale
without maintenance is not a substitute for it.

\begin{table}[h]
\centering
\footnotesize
\caption{ITT and measured spend for all seven policies in five settings;
dashes, spend not instrumented in that wave.}
\label{tab:abas}
\input{sections/appendix/tables/tab_e_basegrid}
\end{table}

\subsection{The Coverage-Gap Ladder}
\label{app:res-covgap}

\textbf{The matched conditional gap is positive in every item class, and
the ordering premium doubles from the base to the high tier.}
Table~\ref{tab:acov} is the high-tier detail behind the kill-switch
estimand. The eleven-seed matched gap reads \ApCondMatched{}, the
eight-seed gate cohort \ApCondEight{}, same sign and overlapping, so the
effect is robust to the seed schedule. Broken out by item class, every
class is positive, from \ApClassTop{}\,pp on the fast on-path class down
to \ApClassFloor{}\,pp at the floor: the priced scheduler does not buy
its conditional gains in one corner of the roster. The ordering premium
itself moves with the dial, $\Delta_{\mathrm{ord}}$ rising from
\ApDordMid{}\,pp at the base tier to \ApDordHigh{}\,pp at the high tier,
a paired difference of \ApGapPriceDiff{} on the five shared seeds.

\begin{table}[h]
\centering
\small
\caption{High-tier arms and the matched pairs, by item class.}
\label{tab:acov}
\input{sections/appendix/tables/tab_e_covgap}
\end{table}

\subsection{Ablations, Per Arm}
\label{app:res-ablation}

\textbf{All seven removals cost success in the predicted direction, and
beliefs cost most.}
Table~\ref{tab:aabl} is the full component panel with paired intervals.
The deepest cut is not a scheduling mechanism: blinding the belief
(confirmations and refutations weighted alike) costs \ApAbDeepest{}\,pp,
more than twice any single index term, because a scheduler that orders
its doubt badly buys the wrong errands. Forcing $\mu_a{:=}0$, the
corpus's implicit no-return assumption, costs \ApRevive{}, the relearning
premium of Appendix~\ref{app:forms}; gating the comparator in abeyance,
the read-time--filter move, costs a significant share on its own, the
deployment twin of the numerical ladder in Table~\ref{tab:bnum}. Every
point estimate matches its preregistered sign; four of seven resolve at
$95\%$ under three seeds, and the direction column reports both facts
without pooling them.

\begin{table}[h]
\centering
\small
\caption{Single-component ablations at the base cap, paired with the full
arm.}
\label{tab:aabl}
\input{sections/appendix/tables/tab_e_ablation}
\end{table}

\subsection{The Allocator Ladder}
\label{app:res-allocator}

\textbf{The numeric index leads exactly where the budget is tight.}
Table~\ref{tab:aall} compares the deployed index against the closed-form
Whittle allocator and a myopic value-rate rule. At the tight cap the
ladder is strict: Whittle gives up \ApAlWhitTight{}\,pp and the myopic
rule \ApAlMyoTight{}\,pp against the numeric index, which prices the free
channel the closed form cannot see. At the loose cap the three run nearly
parallel (\ApAlWhitLoose{}\,pp), which is the zero-shadow-price scope
condition of Proposition~\ref{prop:shadow} showing up on the allocator
axis: with slack budget there is little left for ordering to decide.

\begin{table}[h]
\centering
\small
\caption{Allocator comparison at the tight and loose caps: three seeds,
allocators paired within one wave.}
\label{tab:aall}
\input{sections/appendix/tables/tab_e_allocator}
\end{table}

\subsection{The Coverage Grid and the Zero-Shadow-Price Layer}
\label{app:res-lifeboat}

\textbf{The ordering premium rises monotonically with the coverage gap,
and vanishes the moment the budget stops binding.}
Table~\ref{tab:alif} sweeps coverage $m$, cost tail, budget layer, and
comparator channel on the simulation grid, twenty seeds per cell;
Figure~\ref{fig:aprobe}a draws the four layers. On the
binding, open-channel layer the premium climbs from \ApLifeboatSat{}\,pp
at full coverage to \ApLifeboatMaxGap{} at zero coverage, Spearman
$\rho = \ApLifeboatRho$ across the $m$ ladder. On the non-binding layer
the premium pools to \ApLifeboatNB{}\,pp: nothing, as
Proposition~\ref{prop:shadow} says it must be. The gated-comparator
columns run systematically below their open twins, the grid-level echo of
the comparator ablation. Together the two layers draw the deployment
rule: paid ordering is the shadow price of the knowledge free evidence
cannot reach, and where receipts are plentiful or budget is slack, its
price is zero.

\begin{table}[h]
\centering
\footnotesize
\caption{Ordering premium $\Delta_{\mathrm{ord}}$ (pp) over coverage,
cost tail, budget layer, and channel.}
\label{tab:alif}
\input{sections/appendix/tables/tab_e_lifeboat}
\end{table}

\begin{figure}[t]
\centering
\begin{subfigure}{0.325\linewidth}
  \centering
  \includegraphics[width=\linewidth]{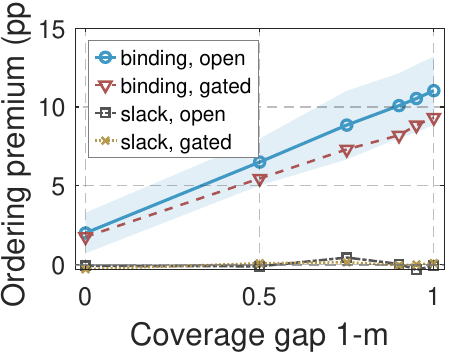}
  \caption{Zero shadow price}
\end{subfigure}\hfill
\begin{subfigure}{0.325\linewidth}
  \centering
  \includegraphics[width=\linewidth]{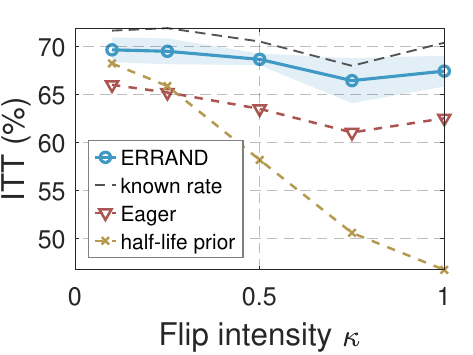}
  \caption{Drift envelope}
\end{subfigure}\hfill
\begin{subfigure}{0.325\linewidth}
  \centering
  \includegraphics[width=\linewidth]{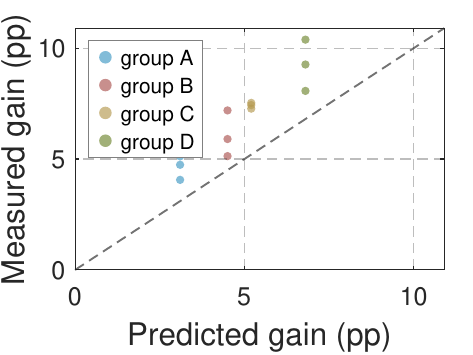}
  \caption{Correlated flips}
\end{subfigure}
\caption{Closing probes. (a)~Ordering premium over the coverage gap:
binding layers climb, slack layers pin to zero; band, twenty-seed s.d.
(b)~ITT across a tenfold drift range; the half-life prior collapses past
the crossover while \sys{} tracks the known-rate reference. (c)~Measured
against predicted coverage gain by flip group; every group sits above the
diagonal.}
\label{fig:aprobe}
\end{figure}

\subsection{The Drift Envelope}
\label{app:res-apiworld}

\textbf{Prices track the world; calendars track their priors.}
Table~\ref{tab:aapi} sweeps the flip intensity on the tool-API world
(Figure~\ref{fig:aprobe}b).
\sys{} holds between \ApAwHatMin{} and \ApAwHatMax\% across a tenfold
range, tracking the known-rate reference without being told $\kappa$. The
literature policy starts competitive at \ApAwLitPeak\% and collapses to
\ApAwLitFloor\% as real drift decouples from its half-life priors,
crossing below even eager revalidation at the fast end: past the
crossover, a wrong calendar is worse than no calendar.

\begin{table}[h]
\centering
\small
\caption{ITT across drift intensity on the tool-API world.}
\label{tab:aapi}
\input{sections/appendix/tables/tab_e_apiworld}
\end{table}

\subsection{Correlated Flips}
\label{app:res-corrflip}

\textbf{Under correlated drift the closed forms underpromise, so the
declared independence is a conservative bound, not a hidden subsidy.}
Table~\ref{tab:acor} runs the preregistered correlated-flip groups and
compares the coverage gain the closed forms predict with what the runs
deliver (Figure~\ref{fig:aprobe}c). Every group lands above its own forecast, pooling to an excess
of \ApCorrExcess{}\,pp: correlated flips make co-located doubt cheaper to
resolve than the independent model books, and en-route receipts repair
neighbors the forecast counts as unresolved.

\begin{table}[h]
\centering
\small
\caption{Predicted versus measured coverage gain by flip group.}
\label{tab:acor}
\input{sections/appendix/tables/tab_e_corrflip}
\end{table}

\subsection{The Regime Axes: Capability and Closure}
\label{app:res-regime}

\textbf{The backbone scales the gain; instrumental closure sets where any
gain can exist.}
Table~\ref{tab:abkb} repeats the headline comparison on the 4B floor: the
gap compresses from \ApFrontierGap{} at 8B to \ApFourBGap{} at 4B, same
sign, one fifth the size. Capability is a gain knob, not a reach knob:
the smaller model converts each maintained item into less recovered
success, but the ordering still pays. Reach has its own boundary.
Table~\ref{tab:aclo} runs the completion-verb layer as a single factor:
with fetch orders off, off-path fetch-verb atoms record \ApClosureOff{}
uses per order, reproducing the zero-use regime the calibration pilot
uncovered, and turning the orders on releases them to \ApClosureOn{}.
Outside the task graph's closure, usage is zero, the index multiplies by
zero, and no scheduler can matter; the regime map's third axis is this
boundary, measured.

\begin{table}[h]
\centering
\small
\caption{Capability floor: the headline arms at 4B beside their 8B
reference.}
\label{tab:abkb}
\input{sections/appendix/tables/tab_e_backbone}
\end{table}

\begin{table}[h]
\centering
\small
\caption{Instrumental closure: uses per order with fetch orders off and
on.}
\label{tab:aclo}
\input{sections/appendix/tables/tab_e_closure}
\end{table}

\subsection{Horizon}
\label{app:res-horizon}

\textbf{Half the horizon already pays, and the second half pays more.}
Table~\ref{tab:ahor} re-reads every base-cap run at a $500$-step prefix.
The gap is already \ApGapAtFive{}\,pp across the three prefix re-reads
(\ApHatFive{} against \ApZeroFive\%) and reads \ApFrontierGapM{}\,pp on
the five-seed full horizon, a ratio of \ApGapRatio{}: maintenance compounds, because
every item kept fresh keeps paying on every later use, while stale-use
share climbs on the arms that let the store age.

\begin{table}[h]
\centering
\small
\caption{ITT and stale-use share at the $500$-step prefix and the full
horizon.}
\label{tab:ahor}
\input{sections/appendix/tables/tab_e_horizon}
\end{table}

\subsection{Freshness Decouples from Performance}
\label{app:res-stale}

\textbf{The store's freshness and the agent's success are different
quantities, and optimizing the first loses the second.}
Table~\ref{tab:asta} is the stale-use panorama. The oracle filter drives
the share to \ApStaleOracleB\%, the cleanest store on the board, and
still trails every spending policy on ITT: masking staleness is not
repairing it. \sys{} wins ITT everywhere while holding the share at
\ApStaleHatB\%, well above the oracle and well below eager revalidation's
\ApStaleEagerB\%; the bigger unmaintained store ages worst at
\ApStaleBigmemB\%. Table~\ref{tab:arad} closes the loop on the teaser:
the nine radar axes of Figure~\ref{fig:teaser}, as numbers, with the
paired margin positive on every axis.

\begin{table}[h]
\centering
\small
\caption{Stale-use share across settings and the $500$-step prefix.}
\label{tab:asta}
\input{sections/appendix/tables/tab_e_stale}
\end{table}

\begin{table}[h]
\centering
\small
\caption{The nine deployment axes of Figure~\ref{fig:teaser}, as numbers;
three seeds per axis.}
\label{tab:arad}
\input{sections/appendix/tables/tab_e_radar}
\end{table}

\subsection{Overhead, Itemized}
\label{app:res-overhead}

\textbf{Scheduling is arithmetic, and targeted errands replace repeated
failure recovery.}
Table~\ref{tab:aovh} is the full per-step account behind the main text's
overhead table: prompt and generation tokens, model calls, scheduler
wall-clock, and per-errand tokens, per policy. The scheduler's own cost
is milliseconds of arithmetic with no model call at any step; the
spending baselines pay more calls per step because unrepaired staleness
converts into retries downstream.

\begin{table}[h]
\centering
\small
\caption{Per-step overhead, all columns, by policy.}
\label{tab:aovh}
\input{sections/appendix/tables/tab_e_overhead}
\end{table}

%% file: sections/appendix/tables/tab_e_frontier.tex
\setlength{\tabcolsep}{3.2pt}
\renewcommand{\arraystretch}{1.12}
\rowcolors{3}{white}{zebra}
\begin{tabular}{@{}llrrrrrrrrr@{}}
\toprule
\textbf{Arm} & \textbf{Cap} & \multicolumn{3}{c}{\cellcolor{grpteal}\textbf{Conditional (\%)}} & \cellcolor{grpblue}\textbf{ITT (\%)} & \multicolumn{2}{c}{\cellcolor{grporange}\textbf{Spend}} & \textbf{Errands} & \textbf{Judged} & \textbf{Cap hit} \\
\cmidrule(lr){3-5}\cmidrule(lr){6-6}\cmidrule(lr){7-8}
 & & judged & full & literal & all steps & \%\,steps & cost/errand & per run & per run & runs \\
\midrule
\textbf{\sys{} ($\hat{\nu}_0$)} & $b{=}2$ & 54.2\stdv{2.5} & 54.1\stdv{2.6} & 88.8\stdv{0.9} & 58.4\stdv{1.0} & 1.9\stdv{0.1} & 1.03 & 32 & -- & 3/3 \\
 & $b{=}3$ & 55.0\stdv{2.3} & 54.8\stdv{2.2} & 89.6\stdv{1.3} & 60.7\stdv{2.3} & 2.8\stdv{0.1} & 0.99 & 50 & -- & 3/3 \\
 & $b{=}6$ & 56.0\stdv{3.3} & 55.8\stdv{3.4} & 87.7\stdv{1.7} & 65.8\stdv{1.5} & 6.2\stdv{0.2} & 1.04 & 102 & 131 & 5/5 \\
 & $b{=}12$ & 57.8\stdv{4.6} & 57.7\stdv{4.6} & 89.6\stdv{1.8} & 71.9\stdv{2.3} & 10.9\stdv{0.2} & 1.09 & 170 & 132 & 5/5 \\
 & $b{=}24$ & 57.9\stdv{4.2} & 57.8\stdv{4.0} & 89.0\stdv{3.3} & 71.5\stdv{2.7} & 10.9\stdv{0.3} & 1.09 & 170 & 120 & 0/5 \\
 & no cap & 58.6\stdv{4.6} & 58.7\stdv{4.7} & 88.1\stdv{1.1} & 71.4\stdv{1.6} & 11.0\stdv{0.1} & 1.07 & 186 & -- & -- \\
\midrule
Eager revalidation ($\nu{\to}0$) & $b{=}2$ & 52.0\stdv{1.2} & 52.0\stdv{1.0} & 88.9\stdv{0.1} & 53.9\stdv{2.2} & 1.9\stdv{0.1} & 1.01 & 35 & -- & 3/3 \\
 & $b{=}3$ & 53.3\stdv{2.9} & 53.2\stdv{2.9} & 87.0\stdv{1.4} & 54.9\stdv{1.0} & 2.9\stdv{0.1} & 1.01 & 47 & -- & 3/3 \\
 & $b{=}6$ & 53.5\stdv{2.8} & 53.4\stdv{2.7} & 87.1\stdv{2.2} & 57.8\stdv{1.5} & 6.2\stdv{0.3} & 1.00 & 105 & 147 & 5/5 \\
 & $b{=}12$ & 56.0\stdv{4.5} & 55.9\stdv{4.4} & 89.3\stdv{1.1} & 61.9\stdv{2.0} & 12.1\stdv{0.1} & 1.01 & 204 & 159 & 5/5 \\
 & $b{=}24$ & 58.0\stdv{3.2} & 58.0\stdv{3.2} & 90.7\stdv{2.3} & 66.2\stdv{3.1} & 24.1\stdv{0.3} & 1.01 & 405 & 132 & 5/5 \\
 & no cap & 59.9\stdv{3.1} & 59.8\stdv{2.9} & 91.0\stdv{0.8} & 67.4\stdv{1.9} & 70.7\stdv{0.1} & 1.00 & 1201 & -- & -- \\
\midrule
No maintenance ($\nu{=}\infty$) & -- & 48.9\stdv{2.3} & 48.8\stdv{2.3} & 78.7\stdv{1.3} & 51.1\stdv{2.3} & 0.0\stdv{0.0} & 0.00 & 0 & 107 & -- \\
\bottomrule
\end{tabular}

%% file: sections/appendix/tables/tab_e_nsr.tex
\setlength{\tabcolsep}{4.5pt}
\renewcommand{\arraystretch}{1.12}
\rowcolors{3}{white}{zebra}
\begin{tabular}{@{}llrrrrr@{}}
\toprule
\textbf{Arm} & \textbf{Cap} & \multicolumn{5}{c}{\cellcolor{grppurple}\textbf{Share of no-spend steps (\%)}} \\
\cmidrule(lr){3-7}
 & & no candidate & index at zero & gate below wage & cap reached & in flight \\
\midrule
\textbf{\sys{} ($\hat{\nu}_0$)} & $b{=}2$ & 10.9 & 20.8 & 17.9 & 48.3 & 2.2 \\
 & $b{=}3$ & 12.1 & 21.6 & 18.2 & 45.0 & 3.1 \\
 & $b{=}6$ & 11.1 & 21.7 & 18.3 & 46.2 & 2.7 \\
 & $b{=}12$ & 12.2 & 39.0 & 33.5 & 12.3 & 2.9 \\
 & $b{=}24$ & 14.5 & 43.7 & 37.5 & 1.0 & 3.3 \\
 & no cap & 14.7 & 41.9 & 39.1 & 0.0 & 4.3 \\
\midrule
Eager revalidation ($\nu{\to}0$) & $b{=}2$ & 10.4 & 8.8 & 4.2 & 73.8 & 2.9 \\
 & $b{=}3$ & 11.7 & 10.3 & 3.1 & 71.6 & 3.3 \\
 & $b{=}6$ & 12.3 & 10.8 & 3.8 & 70.2 & 2.8 \\
 & $b{=}12$ & 16.3 & 14.8 & 5.9 & 59.0 & 3.9 \\
 & $b{=}24$ & 28.5 & 20.8 & 9.0 & 34.5 & 7.1 \\
 & no cap & 48.5 & 27.2 & 11.1 & 0.0 & 13.2 \\
\bottomrule
\end{tabular}

%% file: sections/appendix/tables/tab_e_basegrid.tex
\setlength{\tabcolsep}{2.4pt}
\renewcommand{\arraystretch}{1.12}
\begin{tabular}{@{}lrrrrrrrrrr@{}}
\toprule
\textbf{Method} & \multicolumn{2}{c}{\cellcolor{grpblue}\textbf{$b{=}6$}} & \multicolumn{2}{c}{\cellcolor{grpblue}\textbf{$b{=}12$}} & \multicolumn{2}{c}{\cellcolor{grpblue}\textbf{$b{=}24$}} & \multicolumn{2}{c}{\cellcolor{grpblue}\textbf{high gap}} & \multicolumn{2}{c}{\cellcolor{grpblue}\textbf{4B floor}} \\
\cmidrule(lr){2-3}\cmidrule(lr){4-5}\cmidrule(lr){6-7}\cmidrule(lr){8-9}\cmidrule(lr){10-11}
 & ITT & Spend & ITT & Spend & ITT & Spend & ITT & Spend & ITT & Spend \\
\midrule
No maintenance ($\nu{=}\infty$) & 51.1 & 0.0 & 51.1 & 0.0 & 51.1 & 0.0 & 45.4 & 0.0 & 43.8 & 0.0 \\
Bigger memory, no maint. & 51.2 & 0.0 & 51.2 & 0.0 & 51.2 & 0.0 & 46.5 & 0.0 & 43.6 & 0.0 \\
Fixed-cadence recheck & 58.1 & 5.9 & 61.4 & 10.8 & 64.9 & 11.3 & 56.5 & 10.9 & 46.7 & 10.6 \\
Equal-budget random & 55.1 & 5.8 & 60.8 & 10.6 & 66.2 & 10.8 & 54.4 & 10.5 & 46.4 & 10.5 \\
Eager revalidation ($\nu{\to}0$) & 57.8 & 6.2 & 61.9 & 12.1 & 66.2 & 24.1 & 57.2 & 12.2 & 48.5 & -- \\
\textit{Oracle detect-and-filter}$^{\dagger}$ & \textit{53.8} & 0.0 & \textit{53.8} & 0.0 & \textit{53.8} & 0.0 & \textit{47.3} & 0.0 & \textit{45.1} & 0.0 \\
\rowcolor{herobg} \textbf{\sys{} ($\hat{\nu}_0$)} & \textbf{65.8} & 6.2 & \textbf{71.9} & 10.9 & \textbf{71.5} & 10.9 & \textbf{67.1} & 10.7 & \textbf{50.6} & -- \\
\bottomrule
\end{tabular}

%% file: sections/appendix/tables/tab_e_covgap.tex
\setlength{\tabcolsep}{4.5pt}
\renewcommand{\arraystretch}{1.12}
\rowcolors{2}{white}{zebra}
\begin{tabular}{@{}lrrrrr@{}}
\toprule
\textbf{Arm (high tier)} & \textbf{Cond.\ judged} & \textbf{Cond.\ full} & \textbf{ITT} & \textbf{Spend} & $\Delta_{\mathrm{ord}}$ (pp) \\
\midrule
\rowcolor{herobg} \textbf{\sys{} ($\hat{\nu}_0$)} & 59.5\stdv{2.4} & 59.4\stdv{2.5} & 67.1\stdv{1.3} & 10.7\stdv{0.2} & 7.3\stdv{1.0} \\
Eager revalidation ($\nu{\to}0$) & 54.6\stdv{1.5} & 54.5\stdv{1.4} & 57.2\stdv{0.6} & 12.2\stdv{0.4} & 3.3\stdv{0.7} \\
\midrule
\multicolumn{6}{@{}l}{\emph{Matched conditional pairs by item class (1016 pairs, 11 seeds; seed-clustered):}} \\
\midrule
\quad A\_onpath & \multicolumn{2}{r}{261 pairs} & \multicolumn{3}{r}{$\Delta$cond $+8.4$\,pp [+6.1, +11.1]} \\
\quad B\_onpath & \multicolumn{2}{r}{248 pairs} & \multicolumn{3}{r}{$\Delta$cond $+4.6$\,pp [+1.9, +7.2]} \\
\quad C\_fetch & \multicolumn{2}{r}{256 pairs} & \multicolumn{3}{r}{$\Delta$cond $+3.3$\,pp [-1.2, +7.8]} \\
\quad D\_go & \multicolumn{2}{r}{251 pairs} & \multicolumn{3}{r}{$\Delta$cond $+3.2$\,pp [-0.7, +7.6]} \\
\bottomrule
\end{tabular}

%% file: sections/appendix/tables/tab_e_ablation.tex
\setlength{\tabcolsep}{4.0pt}
\renewcommand{\arraystretch}{1.12}
\rowcolors{2}{white}{zebra}
\begin{tabular}{@{}lrrrrl@{}}
\toprule
\textbf{Variant} & \textbf{ITT (\%)} & \textbf{Cond.\ (\%)} & \textbf{Spend (\%)} & $\Delta$\textbf{ITT [95\% CI]} & \textbf{Predicted} \\
\midrule
\rowcolor{herobg} \textbf{Full \sys{}} & 71.3\stdv{2.0} & 61.3\stdv{0.7} & 10.3\stdv{0.4} & -- & \\
\midrule
w/o deadband gate (pure ranking) & 67.8\stdv{0.7} & 58.3\stdv{1.7} & 10.8\stdv{0.3} & $-3.5$\,pp [-7.4, +0.4] & $-$\,\checkmark \\
coverage sum $\to$ max & 68.1\stdv{2.9} & 58.5\stdv{1.1} & 10.7\stdv{0.5} & $-3.2$\,pp [-6.7, +0.4] & $-$\,\checkmark \\
index: belief only & 66.0\stdv{2.7} & 57.1\stdv{2.0} & 11.0\stdv{0.2} & $-5.3$\,pp [-7.1, -3.5] & $-$\,\checkmark \\
index: age only & 64.7\stdv{4.3} & 57.0\stdv{1.9} & 11.0\stdv{0.2} & $-6.6$\,pp [-13.3, +0.1] & $-$\,\checkmark \\
$\mu_a{:=}0$ (no regime return) & 66.3\stdv{3.0} & 57.9\stdv{3.6} & 11.0\stdv{0.2} & $-5.0$\,pp [-8.6, -1.4] & $-$\,\checkmark \\
uncalibrated belief & 58.3\stdv{3.2} & 52.6\stdv{0.5} & 10.8\stdv{0.4} & $-13.0$\,pp [-17.5, -8.4] & $-$\,\checkmark \\
comparator gated in abeyance & 68.2\stdv{2.4} & 58.2\stdv{1.8} & 10.7\stdv{0.6} & $-3.1$\,pp [-5.2, -1.0] & $-$\,\checkmark \\
\bottomrule
\end{tabular}

%% file: sections/appendix/tables/tab_e_allocator.tex
\setlength{\tabcolsep}{5.5pt}
\renewcommand{\arraystretch}{1.12}
\rowcolors{2}{white}{zebra}
\begin{tabular}{@{}lrrrr@{}}
\toprule
\textbf{Allocator} & \multicolumn{2}{c}{\cellcolor{grpblue}$b{=}6$ (tight)} & \multicolumn{2}{c}{\cellcolor{grpteal}$b{=}24$ (loose)} \\
\cmidrule(lr){2-3}\cmidrule(lr){4-5}
 & ITT (\%) & $\Delta$ vs.\ numeric & ITT (\%) & $\Delta$ vs.\ numeric \\
\midrule
\rowcolor{herobg} \textbf{numeric index (receipts priced in)} & \textbf{67.1}\stdv{1.4} & -- & \textbf{72.5}\stdv{1.5} & -- \\
Whittle closed form & 65.0\stdv{0.8} & $-2.2$ & 70.6\stdv{1.0} & $-1.9$ \\
myopic $\mathrm{VOI}\hat{u}/c$ & 62.3\stdv{0.8} & $-4.8$ & 70.7\stdv{1.3} & $-1.8$ \\
\bottomrule
\end{tabular}

%% file: sections/appendix/tables/tab_e_lifeboat.tex
\setlength{\tabcolsep}{4.2pt}
\renewcommand{\arraystretch}{1.12}
\rowcolors{3}{white}{zebra}
\begin{tabular}{@{}lrrrrrrrr@{}}
\toprule
\textbf{Coverage $m$} & \multicolumn{4}{c}{\cellcolor{grpblue}\textbf{Binding budget}} & \multicolumn{4}{c}{\cellcolor{grporange}\textbf{Non-binding budget}} \\
\cmidrule(lr){2-5}\cmidrule(lr){6-9}
 & \multicolumn{2}{c}{open channel} & \multicolumn{2}{c}{gated comparator} & \multicolumn{2}{c}{open channel} & \multicolumn{2}{c}{gated comparator} \\
 & light $c$ & heavy $c$ & light $c$ & heavy $c$ & light $c$ & heavy $c$ & light $c$ & heavy $c$ \\
\midrule
$1$ & 1.7\stdv{1.2} & 2.3\stdv{1.3} & 1.6\stdv{1.0} & 1.9\stdv{1.1} & -0.1\stdv{0.9} & -0.0\stdv{0.8} & -0.4\stdv{0.9} & -0.2\stdv{0.9} \\
$0.5$ & 5.6\stdv{1.0} & 7.4\stdv{1.3} & 4.8\stdv{0.8} & 6.1\stdv{1.1} & -0.3\stdv{0.9} & 0.1\stdv{0.9} & 0.1\stdv{0.8} & 0.1\stdv{1.0} \\
$0.25$ & 7.2\stdv{1.4} & 10.5\stdv{1.4} & 6.0\stdv{1.0} & 8.6\stdv{0.9} & 0.5\stdv{0.7} & 0.4\stdv{1.1} & 0.1\stdv{0.7} & 0.2\stdv{0.7} \\
$0.1$ & 8.4\stdv{1.2} & 11.8\stdv{1.1} & 6.6\stdv{1.1} & 9.8\stdv{1.0} & -0.1\stdv{1.1} & 0.2\stdv{0.8} & 0.0\stdv{1.0} & -0.1\stdv{0.7} \\
$0.05$ & 8.8\stdv{1.3} & 12.3\stdv{1.2} & 7.0\stdv{1.0} & 10.7\stdv{1.2} & -0.1\stdv{1.1} & -0.5\stdv{1.2} & 0.2\stdv{1.0} & -0.2\stdv{1.0} \\
$0$ & 9.1\stdv{0.9} & 12.9\stdv{1.1} & 7.7\stdv{1.0} & 10.9\stdv{1.0} & 0.1\stdv{0.7} & -0.2\stdv{0.8} & 0.2\stdv{0.7} & -0.0\stdv{0.8} \\
\bottomrule
\end{tabular}

%% file: sections/appendix/tables/tab_e_apiworld.tex
\setlength{\tabcolsep}{5.5pt}
\renewcommand{\arraystretch}{1.12}
\rowcolors{2}{white}{zebra}
\begin{tabular}{@{}lrrrrr@{}}
\toprule
\textbf{Policy} & $\kappa{=}0.1$ & $\kappa{=}0.25$ & $\kappa{=}0.5$ & $\kappa{=}0.75$ & $\kappa{=}1$ \\
\midrule
\rowcolor{herobg} \textbf{\sys{} ($\hat{\nu}_0$)} & 69.7\stdv{1.3} & 69.5\stdv{1.3} & 68.7\stdv{0.6} & 66.5\stdv{2.4} & 67.5\stdv{1.6} \\
\textit{Known-rate reference} & 71.7\stdv{1.6} & 72.0\stdv{0.9} & 70.6\stdv{0.2} & 68.0\stdv{1.3} & 70.4\stdv{2.5} \\
Eager ($\nu{\to}0$) & 66.0\stdv{0.5} & 65.2\stdv{0.2} & 63.5\stdv{0.3} & 61.1\stdv{1.0} & 62.5\stdv{2.6} \\
Literature half-life policy & 68.3\stdv{0.8} & 65.9\stdv{1.4} & 58.2\stdv{0.3} & 50.6\stdv{1.2} & 46.7\stdv{2.5} \\
\bottomrule
\end{tabular}

%% file: sections/appendix/tables/tab_e_corrflip.tex
\setlength{\tabcolsep}{5.5pt}
\renewcommand{\arraystretch}{1.12}
\rowcolors{2}{white}{zebra}
\begin{tabular}{@{}lrrr@{}}
\toprule
\textbf{Flip group} & \textbf{Predicted gain (pp)} & \textbf{Measured (pp)} & \textbf{Excess (pp)} \\
\midrule
flip A & 3.10\stdv{0.00} & 4.64\stdv{0.54} & 1.54\stdv{0.54} \\
flip B & 4.50\stdv{0.00} & 6.07\stdv{1.04} & 1.57\stdv{1.04} \\
flip C & 5.20\stdv{0.00} & 7.41\stdv{0.13} & 2.21\stdv{0.13} \\
flip D & 6.80\stdv{0.00} & 9.24\stdv{1.16} & 2.44\stdv{1.16} \\
\midrule
pooled & 4.90\stdv{1.39} & 6.84\stdv{1.91} & \textbf{1.94\stdv{0.81}} \\
\bottomrule
\end{tabular}

%% file: sections/appendix/tables/tab_e_backbone.tex
\setlength{\tabcolsep}{5.0pt}
\renewcommand{\arraystretch}{1.12}
\rowcolors{2}{white}{zebra}
\begin{tabular}{@{}lrrrr@{}}
\toprule
\textbf{Arm} & \multicolumn{2}{c}{\cellcolor{grppurple}\textbf{4B backbone}} & \multicolumn{2}{c}{\cellcolor{grpblue}\textbf{8B reference}} \\
\cmidrule(lr){2-3}\cmidrule(lr){4-5}
 & ITT (\%) & Cond.\ (\%) & ITT (\%) & Cond.\ (\%) \\
\midrule
No maintenance ($\nu{=}\infty$) & 43.8\stdv{2.2} & 43.6\stdv{1.3} & 51.1\stdv{2.3} & 48.9\stdv{2.3} \\
Eager revalidation ($\nu{\to}0$) & 48.5\stdv{1.5} & 46.0\stdv{1.3} & 61.9\stdv{2.0} & 56.0\stdv{4.5} \\
\rowcolor{herobg} \textbf{\sys{} ($\hat{\nu}_0$)} & 50.6\stdv{1.3} & 47.1\stdv{1.3} & 71.9\stdv{2.3} & 57.8\stdv{4.6} \\
\bottomrule
\end{tabular}

%% file: sections/appendix/tables/tab_e_closure.tex
\setlength{\tabcolsep}{5.5pt}
\renewcommand{\arraystretch}{1.12}
\rowcolors{2}{white}{zebra}
\begin{tabular}{@{}llrr@{}}
\toprule
\textbf{Atom group} & \textbf{Verb class} & \textbf{Fetch orders off} & \textbf{Fetch orders on} \\
\midrule
C\_obj\_value & fetch & 0.00\stdv{0.00} & 4.26\stdv{0.51} \\
D\_path & go & 1.14\stdv{0.04} & 5.62\stdv{0.35} \\
F\_obj\_value & fetch & 0.00\stdv{0.00} & 0.17\stdv{0.00} \\
\bottomrule
\end{tabular}

%% file: sections/appendix/tables/tab_e_horizon.tex
\setlength{\tabcolsep}{5.5pt}
\renewcommand{\arraystretch}{1.12}
\rowcolors{2}{white}{zebra}
\begin{tabular}{@{}lrrrr@{}}
\toprule
\textbf{Method} & \textbf{ITT@500} & \textbf{ITT@2000} & \textbf{Stale@500} & \textbf{Stale@2000} \\
\midrule
No maintenance ($\nu{=}\infty$) & 60.1\stdv{0.9} & 51.1\stdv{2.3} & 28.7\stdv{1.2} & 43.2\stdv{1.2} \\
Bigger memory, no maint. & 60.6\stdv{1.1} & 51.2\stdv{1.3} & 30.6\stdv{0.6} & 53.1\stdv{1.8} \\
Fixed-cadence recheck & 66.5\stdv{0.6} & 61.4\stdv{0.1} & 10.5\stdv{0.8} & 17.1\stdv{1.9} \\
Equal-budget random & 65.4\stdv{1.5} & 60.8\stdv{1.8} & 10.8\stdv{1.7} & 18.6\stdv{0.8} \\
Eager revalidation ($\nu{\to}0$) & 67.2\stdv{0.5} & 61.9\stdv{2.0} & 13.2\stdv{1.4} & 21.9\stdv{1.0} \\
\rowcolor{herobg} \textbf{\sys{} ($\hat{\nu}_0$)} & 73.4\stdv{1.3} & 71.9\stdv{2.3} & 7.2\stdv{0.3} & 12.2\stdv{1.0} \\
\bottomrule
\end{tabular}

%% file: sections/appendix/tables/tab_e_stale.tex
\setlength{\tabcolsep}{5.5pt}
\renewcommand{\arraystretch}{1.12}
\rowcolors{2}{white}{zebra}
\begin{tabular}{@{}lrrrrrr@{}}
\toprule
\textbf{Method} & $b{=}6$ & $b{=}12$ & $b{=}24$ & high & 4B & $T{=}500$ \\
\midrule
No maintenance ($\nu{=}\infty$) & 43.2\stdv{1.2} & 43.2\stdv{1.2} & 43.2\stdv{1.2} & 54.8\stdv{0.3} & 46.5\stdv{0.4} & 28.7\stdv{1.2} \\
Bigger memory, no maint. & 53.1\stdv{1.8} & 53.1\stdv{1.8} & 53.1\stdv{1.8} & 55.6\stdv{0.4} & 51.4\stdv{0.5} & 30.6\stdv{0.6} \\
Fixed-cadence recheck & 26.5\stdv{0.2} & 17.1\stdv{1.9} & 14.8\stdv{0.8} & 21.5\stdv{1.7} & 18.8\stdv{1.1} & 10.5\stdv{0.8} \\
Equal-budget random & 28.3\stdv{0.8} & 18.6\stdv{0.8} & 15.0\stdv{1.4} & 23.3\stdv{0.9} & 20.0\stdv{1.5} & 10.8\stdv{1.7} \\
Eager revalidation ($\nu{\to}0$) & 31.6\stdv{0.5} & 21.9\stdv{1.0} & 13.9\stdv{0.4} & 26.1\stdv{0.6} & 24.9\stdv{1.7} & 13.2\stdv{1.4} \\
\textit{Oracle detect-and-filter}$^{\dagger}$ & 3.4\stdv{1.3} & 3.4\stdv{1.3} & 3.4\stdv{1.3} & 2.6\stdv{1.3} & 4.3\stdv{0.7} & -- \\
\rowcolor{herobg} \textbf{\sys{} ($\hat{\nu}_0$)} & 20.1\stdv{2.0} & 12.2\stdv{1.0} & 12.9\stdv{1.5} & 16.2\stdv{0.5} & 12.9\stdv{1.7} & 7.2\stdv{0.3} \\
\bottomrule
\end{tabular}

%% file: sections/appendix/tables/tab_e_radar.tex
\setlength{\tabcolsep}{5.5pt}
\renewcommand{\arraystretch}{1.12}
\rowcolors{2}{white}{zebra}
\begin{tabular}{@{}llrrr@{}}
\toprule
\textbf{Axis} & \textbf{Family} & \sys{} \textbf{ITT} & \textbf{Eager ITT} & $\Delta$ (pp) \\
\midrule
Budget $b{=}6$, base gap & Budget & \textbf{66.1\stdv{1.4}} & 57.7\stdv{1.7} & $+8.3$ \\
Budget $b{=}12$, base gap & Budget & \textbf{71.0\stdv{1.4}} & 61.7\stdv{0.6} & $+9.4$ \\
Budget $b{=}24$, base gap & Budget & \textbf{72.4\stdv{2.6}} & 66.5\stdv{1.8} & $+5.9$ \\
Budget $b{=}6$, high gap & Coverage & \textbf{60.0\stdv{0.9}} & 53.5\stdv{1.0} & $+6.5$ \\
Budget $b{=}12$, high gap & Coverage & \textbf{66.8\stdv{1.4}} & 57.5\stdv{0.6} & $+9.3$ \\
Budget $b{=}24$, high gap & Coverage & \textbf{66.7\stdv{1.4}} & 61.9\stdv{1.7} & $+4.8$ \\
Capability floor (4B) & Model & \textbf{50.6\stdv{1.3}} & 48.5\stdv{1.5} & $+2.1$ \\
Short horizon ($T{=}500$) & World & \textbf{73.4\stdv{1.3}} & 67.2\stdv{0.5} & $+6.2$ \\
Tool-API world (pooled $\kappa$) & World & \textbf{68.4\stdv{0.7}} & 63.7\stdv{0.7} & $+4.7$ \\
\bottomrule
\end{tabular}

%% file: sections/appendix/tables/tab_e_overhead.tex
\setlength{\tabcolsep}{3.2pt}
\renewcommand{\arraystretch}{1.12}
\rowcolors{2}{white}{zebra}
\begin{tabular}{@{}lrrrrr@{}}
\toprule
\textbf{Policy} & \textbf{Prompt tok/step} & \textbf{Gen.\ tok/step} & \textbf{LLM calls/step} & \textbf{Sched.\ ms/step} & \textbf{Tok/errand} \\
\midrule
\rowcolor{herobg} \textbf{\sys{} ($\hat{\nu}_0$)} & 431\stdv{12} & 64\stdv{2} & 1.10\stdv{0.04} & 39.8\stdv{1.1} & 361\stdv{4} \\
Eager revalidation ($\nu{\to}0$) & 455\stdv{3} & 72\stdv{4} & 1.40\stdv{0.11} & 32.7\stdv{0.8} & 334\stdv{16} \\
Fixed-cadence recheck & 417\stdv{9} & 60\stdv{2} & 1.53\stdv{0.08} & 30.4\stdv{0.4} & 355\stdv{13} \\
Equal-budget random & 426\stdv{10} & 62\stdv{2} & 1.44\stdv{0.06} & 28.8\stdv{2.1} & 361\stdv{11} \\
\bottomrule
\end{tabular}

%% file: sections/appendix/app_f_instruments.tex
\section{Instrument Specifications}
\label{app:instruments}

This appendix reproduces the instruments the study runs on: the briefing
item schema, the work-order contract, the spend ledger, the stratified
accounting audit, and the operational form of every baseline. All
specifications are the pinned versions registered in the instrument
manifest before any full-wave GPU hour, and ship with the release package
of Appendix~\ref{app:repro}.

\subsection{The Briefing Item}
\label{app:item-schema}

Every consolidated item grounds in one checkable atom; instrumentation
keeps items single-atom by design (Section~\ref{app:drift}). The store
carries, per item, exactly the fields below; the worked example is the
item that Figure~\ref{fig:case} and Appendix~\ref{app:cases} follow.

\begin{custombox}{Briefing item schema, with the case-study item}
\setlength{\parindent}{0pt}
\begin{simplecode}
item \{\\
\hspace*{1.5em}atom\_id:\ \ bfB\_val\_oil\_bomb\\
\hspace*{1.5em}group:\ \ \ \ B \ (on-path, slow schedule)\\
\hspace*{1.5em}verb:\ \ \ \ \ buy\\
\hspace*{1.5em}site:\ \ \ \ \ area\_bastion\_watchtower \ (npc\_kael\_ashworth)\\
\hspace*{1.5em}cost:\ \ \ \ \ c = 3 \ (= distance 2 + 1, audited)\\
\hspace*{1.5em}recorded:\ required\_fuse = short\\
\hspace*{1.5em}state:\ \ \ \ in-service | abeyant | superseded(version)\\
\hspace*{1.5em}belief:\ \ \ q, \ receipts N, \ resolve-by, \ last receipt step\\
\}
\end{simplecode}
\textbf{Note:} the recorded value is the handed-over belief, not the live
truth; the world flips \texttt{required\_fuse} on its own schedule, and
nothing in the stream announces it.
\end{custombox}

\subsection{The Work-Order Contract}
\label{app:order-contract}

Orders are the dispatch layer of the exogeneity design
(Section~\ref{app:exo}); their text is whitelisted and identical across
arms.

\begin{custombox}{Work-order fields and the truth-leak whitelist}
\setlength{\parindent}{0pt}
\textbf{Fields:} an order names its verb (\texttt{go} or \texttt{fetch}),
its target site or object, and, for fetch orders, the drop point or
character taking delivery. Nothing else.\\[3pt]
\textbf{Banned by whitelist:} any truth value, any freshness or flip
hint, any belief or schedule field, any mention of maintenance. The
briefing is the only knowledge channel.\\[3pt]
\textbf{Enforcement:} \texttt{assert\_no\_truth\_leak} runs at the first
step of every run and halts on any text outside the pinned whitelist;
the order tables of both dispatch tiers are preregistered and
fingerprinted (Appendix~\ref{app:repro}).
\end{custombox}

\subsection{The Spend Ledger}
\label{app:ledger-spec}

Every run emits a per-step ledger; the no-spend clause column is the
instrument behind Table~\ref{tab:ansr} and behind the deadband's
auditability claim. The excerpt is verbatim from a released base-cap
\sys{} run.

\begin{custombox}{Per-step ledger schema, with a released excerpt}
\setlength{\parindent}{0pt}
\textbf{Columns:} step; errand flag; cost this step; trailing spend per
hundred; wage $\wage$; cumulative free receipts; store size; and, on
every no-spend step, one clause from
\{\texttt{no\_candidate}, \texttt{idx\_le\_0}, \texttt{gate\_below\_nu},
\texttt{budget\_cap}, \texttt{inflight}\}.\\[4pt]
\begin{simplecode}
step  errand cost trail/100 wage\ \ \ reason\\
600 \ \ 1\ \ \ \ 1\ \ \ \ 11.0\ \ \ \ \ 0.0260\ --\\
610 \ \ 0\ \ \ \ 0\ \ \ \ 12.0\ \ \ \ \ 0.0275\ idx\_le\_0\\
618 \ \ 1\ \ \ \ 1\ \ \ \ 12.0\ \ \ \ \ 0.0275\ --\\
620 \ \ 0\ \ \ \ 0\ \ \ \ 13.0\ \ \ \ \ 0.0275\ idx\_le\_0\\
631 \ \ 1\ \ \ \ 2\ \ \ \ 10.0\ \ \ \ \ 0.0264\ --\\
680 \ \ 0\ \ \ \ 0\ \ \ \ \ 8.0\ \ \ \ \ 0.0260\ gate\_below\_nu\\
682 \ \ 1\ \ \ \ 1\ \ \ \ \ 8.0\ \ \ \ \ 0.0261\ --\\
692 \ \ 1\ \ \ \ 2\ \ \ \ \ 8.0\ \ \ \ \ 0.0273\ --
\end{simplecode}
\textbf{Reading:} \texttt{idx\_le\_0} and \texttt{gate\_below\_nu} are
chosen restraint; \texttt{budget\_cap} is inability. The two are never
pooled, which is what makes the plateau claim of Table~\ref{tab:cpre}
attributable.
\end{custombox}

\subsection{The Stratified Accounting Audit}
\label{app:audit-spec}

\begin{custombox}{Accounting audit, run before every wave}
\setlength{\parindent}{0pt}
\textbf{Why stratified:} accounting defects concentrate on the treated,
off-path stratum, where the mechanism's value lives, and a global smoke
test is dominated by on-path traffic that cannot see them
(pitfall~5 of Section~\ref{app:pitfalls}).\\[3pt]
\textbf{Three defect families checked, per stratum:} double-billing (one
errand charged at two sites), locked-edge-passable (a closed path
scored as traversable), and encounter-criterion drift (a success counted
into the wrong denominator).\\[3pt]
\textbf{Companion ledgers:} the dual-domain emission of every conditional
estimate (judged and full, Section~\ref{app:estimands}) and the
per-edge flip-interception ledger behind the effective-rate accounting of
Section~\ref{app:drift}.
\end{custombox}

\subsection{Baseline Instruments, Operationally}
\label{app:baseline-specs}

The oracle detect-and-filter reads the harness's flip event log at
retrieval time and masks any currently stale item, at zero action cost;
it is the read-time--filter ceiling and touches nothing else. The
known-rate reference receives the true $\kappa$ of the drift sweep and
otherwise runs the full scheduler. The fixed-cadence recheck scans two
items per period on a clock, blind to belief. Equal-budget random keeps
\sys{}'s realized spend trigger and draws its target uniformly from the
retrieved candidates. The bigger-store anchor adds twelve legacy items;
paired judging stays on the shared twenty-six. The literature policy maps
published half-life priors to refresh cadences per atom class. On the
mechanism axis, the blind-belief arm re-anchors confirmations and
refutations identically; the gated-comparator arm lets the comparator
read only in-service items; the $\mu_a{:=}0$ arm freezes flip-back in the
world model while the world keeps flipping; and the allocator arms swap
the index for the closed-form Whittle rule \citep{NinoMora2026ARXIV} or
the myopic value-rate ratio, leaving everything else fixed.

\subsection{Environment Card}
\label{app:env-card}

\begin{envcardbox}{The errand world}
\textbf{Harness:} AgentOdyssey v0.1.0 (MIT), driven off-ranking with
custom oscillation step rules; rules register by class discovery, so
drift needs no harness modification.\quad
\textbf{Truth channel:} step rules emit every atom flip as an event;
truth at any step is a log lookup, and no model judges outcomes.\quad
\textbf{Stream:} $2{,}000$ steps, $300$ warm-up, dispatch tiers of
Section~\ref{app:world}.\quad
\textbf{Serving:} two identical vLLM servers, frozen Qwen3-8B, paired
same-machine per seed and cap.
\end{envcardbox}

\begin{limitationbox}
Without fetch-completion orders, off-path fetch-verb atoms are outside
the agent's task graph and record zero uses per order
(Table~\ref{tab:aclo}); conclusions about them require the closure layer,
and the roster excludes one atom behind a connectivity guard for exactly
this reason. Path-locking intercepts a share of nominal flips, so all
drift accounting runs at the effective rate with the interception ledger
attached.
\end{limitationbox}

\begin{bestpracticebox}
Runs are comparable only under equal instrument fingerprints: instrument
tree, dispatch table, drift schedule, and seed, written at start-up by
the same script that launches the run. Per-arm smoke covers every
decision path before any wave, and the stratified audit of
Section~\ref{app:audit-spec} runs before, not after, the GPU hours.
\end{bestpracticebox}

%% file: sections/appendix/app_g_cases.tex
\section{Case Studies}
\label{app:cases}

This appendix replays the recorded evidence behind Figure~\ref{fig:case}
at full resolution, one item and one ledger at a time. The paired case
was selected by a preregistered predicate written before extraction: a
same-seed, same-machine pair of runs in which one arm's stale use
directly fails an order while the other holds the same drift in its
deadband and succeeds on the adjacent step. Every timeline below is a
structured excerpt of released run records, keyed by atom, step, and run;
none is reconstructed after the fact. The item under the microscope is
\texttt{bfB\_val\_oil\_bomb}, whose oscillating attribute
\texttt{required\_fuse} flips from \texttt{short} to \texttt{long} at
step $641$; both arms run the base cap on the same briefing and dispatch
table. Figure~\ref{fig:acase} draws both records.

\begin{figure}[t]
\centering
\begin{subfigure}{0.48\linewidth}
  \centering
  \includegraphics[width=\linewidth]{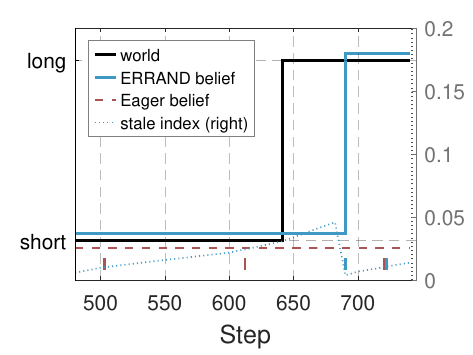}
  \caption{The item, both arms}
\end{subfigure}\hfill
\begin{subfigure}{0.48\linewidth}
  \centering
  \includegraphics[width=\linewidth]{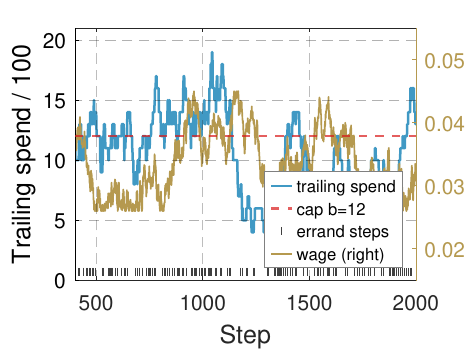}
  \caption{The spend ledger}
\end{subfigure}
\caption{The cases, as recorded. (a)~The item of Figure~\ref{fig:case}:
the world flips at step $641$; \sys{}'s belief follows at $690$ off a
free receipt while eager's never does; dotted vertical, the resolve-by
date; bottom ticks, each arm's recorded events. (b)~A full base-cap
ledger: the trailing account rides the cap early, overrunning only while
an errand is in flight, then the wage, not the cap, sets the pace;
bottom ticks, errand steps.}
\label{fig:acase}
\end{figure}

\subsection{Case 1: The Hold That Paid}
\label{app:case-hold}

\sys{} spends nothing on this item and wins anyway: the deadband holds
the doubt, the resolve-by date bounds the risk, and a co-located errand
returns the new value for free before the task arrives.

\begin{custombox}{Setting: \sys{} arm, base cap, seed 43}
\textbf{Mechanism under test:} the deadband gate with a resolve-by
deadline, and the never-gated comparator that turns a neighboring errand
into a free receipt.
\end{custombox}

\begin{casebox}{successbg}{successframe}{\sys{}: hold, free refresh, success}
\begin{itemize}[nosep, leftmargin=*]
  \item \textbf{Step 600, hold.} The stale index reaches $0.022$, under
  the calibrated gate; the item enters \texttt{hold} with resolve-by
  step $742$. Doubt is now a priced, deadlined state, not a silence.
  \item \textbf{Step 641, the world moves.} \texttt{required\_fuse}
  flips to \texttt{long}. The belief still reads \texttt{short}; the
  index keeps climbing ($0.031 \to 0.046$ by step $682$) but never
  clears the wage: no errand is worth buying yet.
  \item \textbf{Step 690, the free receipt.} A funded errand to the same
  watchtower, whose target is the neighboring item
  \texttt{bfB\_val\_pen}, passes the site; the comparator reads the
  current spec en route and returns \texttt{oil\_bomb}'s new value at
  zero marginal cost. Belief flips to \texttt{long}, state
  \texttt{refresh}, index resets to $0.004$.
  \item \textbf{Step 722, the use.} The order arrives; the served value
  matches the world; the delivery succeeds, budget never spent on this
  item.
\end{itemize}
\vspace{2pt}
\textbf{Verdict:} \textcolor{successframe}{refreshed for free at step
$690$, fifty-two steps after the flip and thirty-two before the use.}\\
\textbf{Downstream:} \textcolor{successframe}{restraint was the winning
allocation; the resolve-by date at $742$ bounded the hold the whole way.}
\end{casebox}

\subsection{Case 2: The Refresh That Didn't}
\label{app:case-eager}

Eager revalidation touches the same item earlier and still fails: an
unpriced budget spreads across the store, and the one refresh that lands
predates the flip it needed to catch.

\begin{custombox}{Setting: eager arm ($\nu{\to}0$), base cap, seed 43}
\textbf{Mechanism under test:} see-doubt-check-doubt revalidation with
no prices, on the identical briefing, world, and machine.
\end{custombox}

\begin{casebox}{failbg}{failframe}{Eager: early refresh, spread budget, stale use}
\begin{itemize}[nosep, leftmargin=*]
  \item \textbf{Step 503, the refresh.} A scheduled errand rechecks
  \texttt{oil\_bomb} and confirms \texttt{short}, correctly: the world
  has not moved yet. The budget bought an answer that was not in doubt.
  \item \textbf{Step 612, the spread.} The window's remaining budget
  goes to the neighboring low-value atom
  \texttt{bfB\_val\_scrap\_iron}; \texttt{oil\_bomb} is not touched
  again.
  \item \textbf{Step 641, the world moves.} The flip lands; the eager
  arm's belief stays \texttt{short}, and unpriced rotation gives this
  item no priority over any other.
  \item \textbf{Step 720, the use.} The paired order arrives two steps
  before Case~1's; the served value is stale, the fuse mismatches, the
  delivery is rejected.
\end{itemize}
\vspace{2pt}
\textbf{Verdict:} \textcolor{failframe}{refreshed at step $503$, $138$
steps too early; never again before the use.}\\
\textbf{Downstream:} \textcolor{failframe}{the run-level ITT gap of
Table~\ref{tab:afro} is this trajectory, repeated: spending more bought
the store less.}
\end{casebox}

\subsection{Case 3: The Inaction Ledger, Read Raw}
\label{app:case-ledger}

The deadband of Case~1 is not a narrative convenience; it is rows in a
ledger. This case reads a released hundred-step window of a base-cap
\sys{} run (the excerpt of Section~\ref{app:ledger-spec} continues here,
and Figure~\ref{fig:acase}b draws the run entire) and attributes every
unspent step.

\begin{custombox}{Setting: \sys{} arm, base cap, seed 42; steps 600--700}
\textbf{Window facts, verbatim from the ledger:} seven errand steps;
every no-spend step carries a clause; trailing spend stays between
$8$ and $14$ per hundred against the cap of $12$.
\end{custombox}

\begin{casebox}{successbg}{successframe}{One window, every silence attributed}
\begin{itemize}[nosep, leftmargin=*]
  \item \textbf{The spends.} Seven errands in the window, costs one to
  two actions each, each preceded by a candidate clearing the wage.
  \item \textbf{The chosen silences.} Seventeen steps log
  \texttt{idx\_le\_0}, certainty at one of the two free ends of the
  index, and thirteen log \texttt{gate\_below\_nu}, live doubt priced
  under the wage: thirty steps of restraint, each with a stated reason.
  \item \textbf{The forced silences.} Six steps log
  \texttt{budget\_cap}, the trailing window momentarily exhausted; one
  logs \texttt{no\_candidate} and one \texttt{inflight}. Restraint and
  starvation stay separable line by line.
  \item \textbf{The run in one row.} Over the full horizon this run
  spends on $9.5\%$ of steps against the $12\%$ allowance, collects
  fourteen free receipts, and holds the wage between $0.026$ and
  $0.047$: the meter, not the cap, sets the pace.
\end{itemize}
\vspace{2pt}
\textbf{Verdict:} \textcolor{successframe}{every unspent step in the
window names its clause; the deadband is an auditable object.}\\
\textbf{Downstream:} \textcolor{successframe}{aggregated over all runs
and caps, these clauses are Table~\ref{tab:ansr}; the plateau reading of
Table~\ref{tab:cpre} rests on exactly this separation.}
\end{casebox}

%% file: sections/appendix/app_h_repro.tex
\section{Reproducibility}
\label{app:repro}

\subsection{Release Package}
\label{app:release}

The release package is the anonymized code supplement accompanying
this submission: the scheduler library and its step rules, the
instrument manifest with both dispatch tables and the drift schedules,
and the briefing roster. No table cell or prose statistic in the main
text or this appendix is hand-entered: each is produced by a generator
script that reads the per-run atoms (run-level results, per-atom
rollups, spend ledgers, flip event logs, and start-up fingerprints)
and writes both the tables and the named macros the prose cites.

\subsection{Compute}
\label{app:compute}

The study ran inside a preregistered envelope of roughly $900$ GPU-hours
(about $38$ H100-days) across four waves; the delivered atoms cover more
than $260$ instrumented GPU deployments of $2{,}000$ steps each
($1{,}700$ scored), plus the zero-GPU re-reads of the $500$-step prefix
and the stale-share rollups. The coverage grid of
Table~\ref{tab:alif} ($960$ cells' worth of simulated deployments), the
two-state value iteration of Table~\ref{tab:bnum}, and the power
forecast of Table~\ref{tab:cpow} are CPU-only and consumed no GPU time.
Serving uses two identical vLLM hosts with the frozen 8B backbone and a
4B sibling for the capability floor.

\subsection{Randomness and Seeds}
\label{app:seeds}

Seeds follow the preregistered ladder: five on the frontier arms' main
surfaces, eight
extending to eleven on the high-gap pair under a preauthorized precision
upgrade, three on baseline arms and breadth surfaces, and twenty per
cell on the simulation grid. The harness's generator is seeded per run, the two arms
of every comparison run on the same machine per seed and cap, and
servers rotate across seeds. Tables report mean and one standard
deviation across runs; paired comparisons inherit the same-machine
design and use seed-level $t$ intervals, with seed-clustered bootstrap
cross-checks on the matched pairs. Every run writes a start-up
fingerprint of its instrument tree, dispatch table, drift schedule, and
seed, and runs enter a comparison only under equal fingerprints. A
fingerprint pins the world side of a run, while sampling on the serving
hosts stays nondeterministic, so a rerun reproduces the seed's world and
the cohort distributions rather than per-run numbers; every claim in
this paper reads seed-level contrasts, never single runs.